%% file: gbrl.tex
\documentclass[10pt, conference]{ieeeconf}      

\include{preamble}

\newcommand \citep {\cite}
\newcommand \citet {\cite}

\usepackage{pgfplots}
\usepackage{tikz}
\pgfplotsset{compat=newest} 
\pgfplotsset{plot coordinates/math parser=false}

\title{\LARGE \bf Monte-Carlo utility estimates for Bayesian reinforcement learning}
\author{Christos Dimitrakakis}

\begin{document} 

\maketitle
\thispagestyle{empty}
\pagestyle{empty}

\begin{abstract} 
  This paper introduces a set of algorithms for Monte-Carlo Bayesian reinforcement learning. Firstly, Monte-Carlo estimation of upper bounds on the Bayes-optimal value function is employed to construct an optimistic policy. Secondly, gradient-based algorithms for approximate upper and lower bounds are introduced. Finally, we introduce a new class of gradient algorithms for Bayesian Bellman error minimisation. We theoretically show that the gradient methods are sound. Experimentally, we demonstrate the superiority of the upper bound method in terms of reward obtained. However, we also show that the Bayesian Bellman error method is a close second, despite its significant computational simplicity.
\end{abstract}

\section{Introduction}
\label{sec:introduction}
Bayesian reinforcement learning~\cite{strens2000bayesian,vlassis2012bayesian} is the decision-theoretic approach~\cite{Degroot:OptimalStatisticalDecisions}
to solving the reinforcement learning problem. Unfonrtunately, calculating posterior distributions can be computationally expensive. Morever, the Bayes-optimal
decision can be  intractable~\cite{NIPS2007:ross:bapomdp,dimitrakakis:icaart2010,strens2000bayesian}, and even calculating an optimal solution in a restricted class can be difficult~\cite{dimitrakakis:mmbi:ewrl:2011}. This paper proposes a set of algorithms that take actions by estimating bounds on the Bayes-optimal utility through sampling. They include a direct Monte-Carlo approach, as well as gradient-based approaches. We demonstrate the effectiveness of the proposed algorithms experimentally.

\subsection{Setting}
\label{sec:setting}
In the \emph{reinforcement learning problem}, an agent is acting in some unknown Markovian environment $\mdp \in \MDPs$, according to some policy $\pol \in \Pols$. The agent's policy is a procedure for selecting actions, with the action at time $t$ being $a_t \in \CA$. The environment reacts to this action with a sequence of states $s_t \in \CS$ and rewards $r_t \in \Reals$. Since the agent may be learning from experience, this interaction may depend on the complete history,  $h_t \in \CH$, where $\CH \defn (\CS \times \CA \times \Reals)^*$ is the set of all state action reward sequences.

The complete \emph {Markov decision process} (MDP) is specified as follows.
The agent's action at time $t$ depends on the history observed so far:
\begin{equation}
  \label{eq:policy}
  a_t \mid h_t = (\seq{s}{t}, \seq{r}{t}, \seq{a}{t-1}) \sim \Pr^\pol(a_t \mid \seq{s}{t}, \seq{r}{t}, \seq{a}{t-1}),
\end{equation}
where $\seq{s}{t}$ is a shorthand for the sequence $(s_i)_{i=1}^t$;
similarly, we use $\pseq{s}{k}{t}$ for $(s_i)_{i=k}^t$. We denote the environment's response at time $t+1$ given the history at time $t$ by:
\begin{equation}
  \label{eq:mdp}
  s_{t+1}, r_{t+1}  \mid  h_t = (\seq{s}{t}, \seq{r}{t}, \seq{a}{t}) \sim \Pr_\mdp(s_{t+1}, r_{t+1} \mid s_t, a_t)
\end{equation}
Finally, the agent's goal is determined through its utility:
\begin{equation}
  \label{eq:utility}
  U = \sum_{t=1}^\infty \disc^{t-1} r_t,
\end{equation}
which is a discounted sum of the total instantaneous rewards obtained, with $\disc \in [0,1]$. Without loss of generality, we assume that $U \in [0, \Urange]$.  The optimal agent policy  maximises $U$ in expectation, i.e.
\begin{equation}
  \label{eq:expected-utility}
  \max_{\pol \in \Pols} \Emp U,
\end{equation}
where  $\Pmp, \Emp$ denote probabilities and expectations under the process jointly specified by $\mdp, \pol$.
However, as in the reinforcement learning problem the environment $\mdp$ is unknown, the above maximisation is ill-posed. Intuitively, the agent can increase its expected utility by either: 
\begin{inparaenum}[(i)]
\item Trying to better estimate $\mdp$ in order to perform the maximisation later (exploration), or
\item Use a best-guess estimate of $\mdp$ to obtain high rewards (exploitation).
\end{inparaenum}

In order to solve this trade-off, we can adopt a Bayesian viewpoint~\cite{Degroot:OptimalStatisticalDecisions,savage1972fs}, where we consider a (potentially infinite) set of environment models $\MDPs$. In particular, we select a \emph{prior probability} measure $\bel$ on $\MDPs$. For an appropriate subset $B \subset \MDPs$, the quantity $\bel(B)$ describes our initial belief that the correct model lies in $B$. We can now formulate the alternative goal of maximising the expected utility with respect to $\bel$:
\begin{equation}
  \label{eq:bayes-utility}
  \E^*_\bel U \defn \max_\pol \Ebp U = \max_\pol \int_\MDPs (\Emp U) \dd{\bel(\mdp)}.
\end{equation}
This makes the problem formally sound. A policy $\pol_\bel^* \in \argmax_\pol \Ebp U$ is called \emph{Bayes-optimal} as it solves the exploration-exploitation problem with respect to out prior belief $\bel$. However, its computation is generally hard~\cite{duff2002olc} even in restricted classes of policies~\cite{dimitrakakis:mmbi:ewrl:2011}. On the other hand, simple heuristics such as Thompson sampling~\cite{thompson1933lou,strens2000bayesian} provide an efficient trade-off~\cite{agrawal:thompson,Kaufmann:Thompson} between exploration and exploitation.

\subsection{Related work and our contribution}
\label{sec:related-work}
One difficulty that arises when adopting a Bayesian approach to sequential decision making is that in many interesting problems, the posterior calculation itself requires approximations, mainly due to partial observability~\citep{Poupart:ModelBayesBAPOMDP,NIPS2007:ross:bapomdp}.
The second and more universal problem, which we consider in this paper, is that finding the Bayes-optimal policy is hard, as the set of policies we must consider grows exponentially with the horizon $T$. However, heuristics exist which, given the current posterior, can obtain a near-optimal policy~\citep{Kolter-Ng:NearBayesianExploration,DBLP:conf/ijcai/CastroP07,strens2000bayesian,araya2012near,dimitrakakis:mmbi:ewrl:2011,poupart2006asd}. In this paper we shall focus on model-based algorithms that use approximate lower and upper bounds on the Bayes-optimal utility to select actions.

The general idea of computing lower and upper bounds via Monte-Carlo sampling in model-based Bayesian reinforcement learning was introduced in~\cite{dimitrakakis:icaart2010}. This sampled MDP models from the current belief to estimate stochastic upper and lower bounds. These bounds were then used to perform a stochastic branch and bound search for an optimal policy. In a follow-up paper~\cite{dimitrakakis:mmbi:ewrl:2011}, an attempt was made to obtain tighter lower bounds by finding a good memoryless policy. An earlier class of approaches involving lower bounds is the work of \cite{poupart2006asd}, which sampled beliefs rather than MDPs to construct lower bound approximations.

In order to perform the approximations, we also introduce a number of gradient-based algorithms. Relevant work in this domain includes the Gaussian process (GP) based algorithms suggested by  \cite{engel2003bayes,ghavamzadeh:bpga} and \cite{Deisenroth2011:pilco}. In particular, \cite{engel2003bayes} performs an incremental temporal-difference fit of the value function using GPs, implicitly using the empirical model of the process. The other two approaches are model-based, with \cite{ghavamzadeh:bpga} estimating a gradient direction for policy improvement by drawing sample trajectories from the marginal distribution. An analytic solution to the problem of policy improvement with GPs is given by~\cite{Deisenroth2011:pilco}, which however relies on the \emph{expected} transition kernel of the process and so does not appear to take the model uncertainty into account.

The approaches suggested in this paper are considerably simpler, as well as more general, in that they are applicable to any Bayesian model of the Markov process and parametrisation of the value function. The fundamental idea stems from the observation that, in order to estimate the Bayes-utility of a policy, we can draw sample MDPs from the posterior, calculate the (either current policy's, or the optimal) utility for each MDP and average. The same effect can be achieved in an iterative procedure, by drawing only one MDP, estimating the utility of our policy, and then adjusting our parameters to approach the sampled utility. This can be achieved with gradient methods. Finally, we  use the same sampling idea to minimise the Bellman error of the Bayes-expected value function, in a fully incremental fashion that explicitly takes into account the model uncertainty.

\section{Gradient Bayesian reinforcement learning}
\label{sec:gbrl}
Imagine that the history $\hist_t \in \CH$ of length $t$ has been generated from $\Pmp$, the process defiend by an MDP $\mdp \in \MDPs$ controlled with a history-dependent policy $\pol$. Now consider a prior belief $\bel_0$ on $\MDPs$ with the property that $\bel_0(\cdot \mid \pol) = \bel_0(\cdot)$, i.e. that the prior is independent of the policy used. Then the posterior probability, given a history $\hist_t$ generated by a policy $\pol$, that $\mdp \in B$ can be written as:
\begin{align}
  \bel_t(B \mid \pol) \defn \bel_0(B \mid \hist_t, \pol) 
  &=
  \frac{\int_B \Pmp (\hist_t) \dd{\bel_0}(\mdp)}
  {\int_\MDPs \Pmp (\hist_t) \dd{\bel_0}(\mdp)}.
  \label{eq:posterior}
\end{align}
Fortunately, the dependence on the policy can be removed, since the posterior is the same for all policies that put non-zero mass on the observed data.
Thus, in the sequel we shall simply write $\bel_t$ for the posterior probability over MDPs at time $t$.

\subsection{Value functions}
\label{sec:value-function}
Value functions are an important concept in reinforcement learning. Briefly, a value function $V_\mdp^\pol : \CS \to \Reals$ gives the expected utility for the policy $\pol$ acting in an MDP $\mdp$, given that we start at state $s$, i.e.
\begin{equation}
  \label{eq:value-function}
  V_\mdp^\pol(s) \defn \Emp(U \mid s_t = s).
\end{equation}
A similar notion is expressed by the $Q$-value function $Q_\mdp^\pol : \CS \times \CA \to \Reals$, which is the expected utility for the policy $\pol$ acting in an MDP $\mdp$, given that we start at state $s$ and take action $a$, i.e.
\begin{equation}
  \label{eq:q-value-function}
  Q_\mdp^\pol(s, a) \defn \Emp(U \mid s_t = s, a_t = a).
\end{equation}
Similarly, and with a slight abuse of notation, we define the \emph{Bayesian} value function $V_\bel^\pol : \CS \to \Reals$, and the related Bayesian $Q$-value function  $Q_\bel^\pol : \CS \times \CA \to \Reals$. These are defined for any belief $\bel$ and policy $\pol$ to be the corresponding expected value functions over all MDPs.
\begin{align}
  \label{eq:bayesian-value-function}
  V_\bel^\pol(s) &\defn \int_\MDPs V_\mdp^\pol(s) \dd{\bel(\mdp)}\\
  Q_\bel^\pol(s,a) &\defn \int_\MDPs Q_\mdp^\pol(s,a) \dd{\bel(\mdp)}
\end{align}
Due to the convexity of the Bayes-optimal expected utility~\cite{Degroot:OptimalStatisticalDecisions} with respect to the belief $\bel$,  it can be bounded from above and below also for the Bayesian RL problem~\cite{dimitrakakis:icaart2010}:
\begin{equation}
  \label{eq:bounds}
  \int_\MDPs \max_\pol (\E_\bel^\pol U) \dd{\bel}(\mdp)
  \geq
   \E_\bel^* (U)
  \geq
  \E_\bel^{\pol'}(U),
  \qquad
  \forall \pol' \in \Pols.
\end{equation}
Since it hard to find the Bayes-optimal policy~\cite{duff2002olc,dearden99bayesian,Degroot:OptimalStatisticalDecisions,dimitrakakis:icaart2010}, we may instead try and estimate upper and lower bounds on the expected utility, and consequently, on the $Q$-value function. These can then be used to implement a heuristic policy that is either exploratory (when we use upper bounds) or conservative (when we use lower bounds).

To achieve this, we propose a number of simple algorithms. 
First, we describe the direct upper bound estimation proposed in \cite{dimitrakakis:icaart2010} in the context of tree search. Here, we apply it to select a policy directly, in a manner similar to the lower bound approach in~\cite{dimitrakakis:mmbi:ewrl:2011}. We then describe gradient-based incremental versions of both algorithms. However, all of these algorithms require estimating the value function of a sampled MDP, a potentially expensive process. For this reason, we also derive a gradient-based algorithm for minimising the Bayes-value function Bellman error. This is shown to perform almost as well as the previous algorithms, with significantly less computational effort.

\subsection{Direct upper bound estimation}

The idea of the following algorithm stems directly from the definition of the upper bound~(\ref{eq:bayes-utility}). In fact, \cite{dimitrakakis:icaart2010} had previously used such upper bounds in order to guide tree search, while \cite{dimitrakakis:mmbi:ewrl:2011} had used \emph{lower bounds} directly for taking actions. 
However, to our knowledge the simple idea of estimating the upper bound (\ref{eq:bayes-utility}) and using it to directly take actions has never been tried before in practice.

 We can estimate an upper bound value vector\footnote{For continuous spaces, this can be defined on a set of representative states.} $\vq$ by direct Monte Carlo sampling\footnote{Due to the Hoeffding bound~\cite{Hoeffding:SumInequalities} and the boundedness of the value function, it is easy to see that this estimate is $O(1/\sqrt{\nsam})$-close to the upper bound (\ref{eq:bounds}) with high probability.}
 from our belief $\bel$:
\begin{equation}
  \label{eq:vq}
  q_{s,a} = \frac{1}{\nsam} \sum_{i=1}^\nsam Q_{\mdp_i}^*(s,a), \qquad  \mdp_i \sim \bel.
\end{equation}
This idea is significantly simpler than that constructing \emph{credible intervals} (see for example \cite{srinivas:gp-bandits:icml2010}). In addition, estimation of  $Q_{\mdp_i}^*$ for each sampled MDP is easy. This is in contrast with the lower bound approach advocated in \cite{dimitrakakis:mmbi:ewrl:2011}. 

\begin{algorithm}
  \begin{algorithmic}
    \STATE \textbf{Input} prior $\bel_0$, value vector $\vq$, initial state $s_0$, number of samples $\nsam$.
    \FOR {$t=0, \ldots$}
    \IF {switch-policy}
    \STATE $\mdp_1, \ldots, \mdp_k \sim \bel_t$ \hfill // Sample $\nsam$ MDPs
    \STATE $\vq = \frac{1}{\nsam} \sum_{i=1}^\nsam Q_{\mdp_i}^*$ \hfill // Get $Q$ upper bound.
    \ENDIF
    \STATE $a_t  = \argmax_{a \in \CA} \vq_{s, a}$. \hfill // Act in the real MDP
    \STATE $s_{t+1}, r_{t+1} \sim \mdp$ \hfill // Observe next state, reward
    \STATE $\bel_{t+1}(\cdot) = \bel_t(\cdot \mid s_{t+1}, r_{t+1}, s_t, a_t)$ \hfill // Update posterior
    \ENDFOR
  \end{algorithmic}
  \caption{U-MCBRL: Upper-bound Monte-Carlo Bayesian RL}
  \label{alg:ubrl}
\end{algorithm}
The algorithm is presented in Alg.\ref{alg:ubrl}.
A hyperparameter of the algorithm is the number of samples $\nsam$ to take at each iteration, as well as the points at which to switch policy\footnote{since re-sampling and calculating new value functions is expensive}. This paper uses the simple strategy of linearly incrementing the switching interval. Let us now see how we can directly approximate both lower bounds such as those in \cite{dimitrakakis:mmbi:ewrl:2011}, and upper bounds, such as this in Alg.~\ref{alg:ubrl}, via gradient methods. 

\subsection{Direct gradient approximation}
\label{sec:direct}
We now present a simple class of algorithms for gradient Bayesian reinforcement learning. First, let us consider the estimation for a specific policy $\pol$, which will correspond to approximating a lower bound. Define a family of functions $\vv_\param : \CS \to \Reals$, $\cset{\vv_\param}{\param \in \Params}$. We consider the problem of estimating the expected value function given some belief $\bel$:
\begin{align}
  \label{eq:estimation}
  \min_{\param \in \Params} & f(\theta),
  &
  f(\theta)&  \defn \int_\CS g(\param; s) \dd{\chi}(s),
  \\ &&
  g(\param; s) & \defn
  \|\vv_\param(s) - V^\pol_\bel(s)\| 
  \label{eq:global-g}
\end{align}
where $\chi$ is a measure on $\CS$, and $\|\cdot\|$ is the Euclidean norm.
Then the derivative of \eqref{eq:global-g} can be written as:
\begin{align}
  \label{eq:cost}
  \nabla_\param g(\param;s)
  & = 
  2\left(\vv_\param(s) - V^\pol_\bel(s)\right)
  \nabla_\param \vv_\param(s).
\end{align}
Let $\omega_k(s) = V_{\mu_k}^\pol(s)$, be the value function of an MDP sampled from the belief, i.e. $\mu_k \sim \bel$. Then, due to the linearity of expectations, it is easy to see that:
\begin{equation}
    \nabla_\param g(\param;s)
    = \E_\bel\left[2\left(\vv_\param(s) - \omega_k(s)\right)
      \nabla_\param \vv_\param(s)\right].
    \label{eq:expected-cost}
\end{equation}
Consequently, $\omega_k$ can be used to obtain the following stochastic approximation~\cite{robbins1951stochastic,BertsekasTsitsiklis:NDP} algorithm
\begin{equation}
  \label{eq:stoch-approx}
  \theta_{k+1} = \theta_k 
  - \step_k \left(\vv_\param(s) - \omega_k(s)\right)
      \nabla_\param \vv_\param(s),
\end{equation}
where $\step_k$ must be a step-size parameter satisfying $\sum_k \step_k = \infty$, $\sum_k \step^2_k < \infty$. A similar approach can be used to estimate the $Q$-value function with an approximation $\vq_\param : \CS \times \CA \to \Reals$:
\begin{equation}
  \label{eq:Q-stoch-approx}
  \theta_{k+1} = \theta_k 
  - \step_k \left(\vq_\param(s,a) - \omega_k(s,a)\right)
      \nabla_\param \vq_\param(s,a),
\end{equation}
where $\omega_k(s, a) = Q_{\mu_k}^\pol(s,a)$. This update can also be performed over the complete state-action space
\begin{align}
  \label{eq:D-stoch-approx}
  \theta_{k+1} &= \theta_k 
  - \step_k \sum_{s,a} D_k(s,a),
  \\
  D_k(s,a) &= \left(\vq_\param(s,a) - \omega_k(s,a)\right) \nabla_\param \vq_\param(s,a).
\end{align}

The same procedure can be applied to approximate the upper bound~(\ref{eq:bounds}). This only requires a trivial modification to the above algorithms, by setting $\omega_k(s) = V^*_{\mu_k}(s)$ or $\omega_k(s,a) = Q^*_{\mu_k}(s,a)$ in either case. It is easy to see that the above approximation still holds.

\begin{algorithm}
  \begin{algorithmic}
    \STATE \textbf{Input} prior $\bel_0$, parameters $\param_0$, initial state $s_0$
    \FOR {$t=0, \ldots$}
    \STATE $\mdp_t \sim \bel_t$ \hfill // Sample an MDP
    \STATE $\omega_t = Q_{\mdp_t}^\pol$ (or $Q_\mdp^*$) \hfill // Get value of sample
    \STATE $\theta_{t+1} = \theta_t
  - \step_k \sum_{s,a} D_k(s,a)$ \hfill // Update parameters
    \STATE $a_t  = \argmax_{a \in \CA} \vq_{\param_t}(s, a)$. \hfill // Act in the real MDP
    \STATE $s_{t+1}, r_{t+1} \sim \mdp$ \hfill // Observe next state, reward
    \STATE $\bel_{t+1}(\cdot) = \bel_t(\cdot \mid s_{t+1}, r_{t+1}, s_t, a_t)$ \hfill // Update posterior
    \ENDFOR
  \end{algorithmic}
  \caption{DGBRL: Direct gradient Bayesian RL.}
  \label{alg:dgbrl}
\end{algorithm}

To make the approximation faster, we can take a single MDP sample at every step, take an action, and then use the previous approximation for the next step. If the belief $\bel_t$ changes sufficiently slowly then this will be almost as good as taking multiple samples and finding the best approximation at every step. The complete algorithm is shown in Alg.\ref{alg:dgbrl}. The advantage of this idea over the upper and lower bound approach advocated in~\cite{dimitrakakis:icaart2010,dimitrakakis:mmbi:ewrl:2011}, is that we can re-use information from previous steps without needing to take multiple MDP samples.

In either case, the computational difficulty is the calculation of $V_{\mu_k}$, which we still need to do once at every step. The next section discusses another idea where the complete estimation of a value function for each sampled MDP is not required.

\subsection{Temporal difference-like error minimisation}
\label{sec:td}
One alternative idea is to simply estimate a consistent 
value function approximation, similar to those used in temporal-difference (TD) methods (in particular the gradient-based view of TD-like methods in~\cite{BertsekasTsitsiklis:NDP}).
The general idea is to form the following minimisation problem:
\begin{align}
  \label{eq:td-estimation}
  \min_{\param \in \Params} & f(\theta),
 \quad
  f(\theta)  \defn \int_\CS g(\param; s) \dd{\chi}(s),
  \\ 
  g(\param; s) & \defn \int_\MDPs \|h(\param; \mdp, s)\| \dd{\bel}(\mdp)
  \\ 
  h(\param; \mdp, s) & \defn  \vv_\param(s) - \rew(s) - \disc  \! \int_{\mathrlap{\CS}} \vv_\param(s') \dd{\Pmp}(s'\mid s).
  \label{eq:td-approx}
\end{align}
Now let us sample a state  $s_k \sim \chi$ from the state distribution, an MDP $\mdp_k \sim \bel$ from the belief and a next state $s'_k \sim \Pr^\pol_{\mdp_k}(s' \mid s_k)$ from the transition kernel of the sampled MDP given the sampled current state. Using the euclidean norm for $\|\cdot\|$  and taking the gradient with respect to $\param$ we obtain:
\begin{align}
  D_k &=
 2h(\param_k; \mdp_k, s_k) \left( \nabla_\param \vv_{\param_k}(s_k) - \disc \nabla_\param \vv_{\param_k}(s_k')\right)
 \label{eq:td-error}
  \\
  \theta_{k+1} &=  \theta_k - \step_k D_k.
  \label{eq:td-gradient}
\end{align}
By choosing an appropriate approximation architecture, e.g. a linear approximation with bounded bases, the following corollary holds:
\begin{corollary}
  If $\|\nabla_{\theta} \vv_\theta\| \leq c$ and  $\|\nabla^2_{\theta} \vv_\theta\| \leq c'$ with $c, c' < \infty$, then  $f(\theta_k)$ converges, with $\lim_{t \to \infty} \nabla_\theta f(\theta_k) = 0$.
  \label{cor:td-error}
\end{corollary}
\begin{proof}
This results follows from Proposition 4.1 in \cite{BertsekasTsitsiklis:NDP}, since the the sequence satisfies the four conditions in Assumption~4.2. (a) $f \geq 0$. (b) $f$ is twice differentiable and its second derivative is bounded, as $\norm{\int_\CS \nabla^2_{\theta} \vv_\theta(s')  \dd{\Pmp}(s'\mid s)}
\leq \int_\CS \|\nabla^2_{\theta} \vv_\theta(s')\|  \dd{\Pmp}(s'\mid s)
\leq c$. (c) By taking expectations over the sample, it is easy to see that $\E D_k = \nabla_\param f(\theta_k)$. (d) follows from the boundedness of the first derivative. 
\end{proof}

\subsection{Bellman error minimisation}
\label{sec:bellman}
An alternative formulation is Bellman error minimisation (\cite{BertsekasTsitsiklis:NDP}, Sec. 6.10), where instead of minimising the error with respect to the current policy, we minimise the error over the Bellman operator applied to the current value function. This is simplest to do when we are working with $Q$-value functions. Then the problem can be written as:
\begin{align}
  \label{eq:bellman-estimation}
  \min_{\param \in \Params} & f(\theta),
  \quad
  f(\theta)  \defn \sum_{a \in \CA} \int_\CS g(\param; s,a) \dd{\chi}(s),
  \\ 
  g(\param; s,a) & \defn \int_\MDPs \|h(\param; \mdp, s,a)\| \dd{\bel}(\mdp)
  \\ 
  h(\param; \mdp, s,a) & \defn  \vq_\param(s,a) - \rew(s) - \disc  \! \int_{\mathrlap{\CS}} \vq_\param(s',a^*(s')) \dd{\Pmp}(s'\mid s)
  \\
  a^*(s') &\defn \argmax_{a' \in \CA} \vq_\param(s', a').
\end{align}
Using the same reasoning as in Sec.~\ref{sec:td}, we sample  $s_k \sim \chi$, $\mdp_k \sim \bel$ from the belief and a next state $s'_k \sim \Pr^\pol_{\mdp_k}(s' \mid s_k)$ from the transition kernel of the sampled MDP given the sampled current state. Using the euclidean norm for $\|\cdot\|$  and taking the gradient with respect to $\param$ we obtain the algorithm:
\begin{align}
  \label{eq:bellman-update}
  D_k & = 2h(\theta_k; \mdp_k, s_k, a_k) [\nabla_{\param_k} \vq_\param(s_k,a_k)
    \\
    &- \disc \nabla_\param \vq_{\param_k}(s'_{k}, a^*(s'_{k}))]\nonumber
  \\
  \theta_{k+1} &= \theta_k - \eta_k D_k.
\end{align}
It easy to see that Corollary~\ref{cor:td-error} is also applicable to this update sequence. When the state sequence is generated from a particular policy, rather than being drawn from some distribution $\chi$, we obtain Alg.\ref{alg:bgbrl}.
\begin{algorithm}
  \begin{algorithmic}
    \STATE \textbf{Input} prior $\bel_0$, parameters $\param_0$, initial state $s_0$
    \FOR {$t=0, \ldots$}
    \STATE $\mdp_t \sim \bel_t$ \hfill // Sample an MDP
    \STATE $s'_t \sim \Pr^\pol_{\mdp_t}(s_{t+1} \mid s_t)$ \hfill // Sample a next state
    \STATE $\theta_{t+1} = \theta_t
  - \step_t D_t$ \hfill // Update parameters using \eqref{eq:bellman-update}
    \STATE $a_t  = \argmax_{a \in \CA} \vq_{\param_t}(s, a)$. \hfill // Act in the real MDP
    \STATE $s_{t+1}, r_{t+1} \sim \mdp$ \hfill // Observe next state, reward
    \STATE $\bel_{t+1}(\cdot) = \bel_t(\cdot \mid s_{t+1}, r_{t+1}, s_t, a_t)$ \hfill // Update posterior
    \ENDFOR
  \end{algorithmic}
  \caption{BGBRL: Bellman gradient Bayesian RL}
  \label{alg:bgbrl}
\end{algorithm}

\section{Experiments}

We present experiments illustrating the performance of U-MCBRL and BGRL and compare them with other algorithms. In particular we also examine the lower-bound algorithm presented in~\cite{dimitrakakis:mmbi:ewrl:2011}, the well known UCRL~\cite{JMLR:UCRL2} algorithm,\footnote{Although UCRL is defined for undiscounted problems, it is trivial to apply to discounted problems by adding replacing  average  value iteration with discounted value iteration.} and $Q(\lambda)$, for completeness. 

\subsection{Experiment design}

Since all algorithms have  hyperparameters, we followed a principled experiment design methodology. 
Firstly, we selected a set of possible hyperparameter values for each algorithm. For each evaluation domain, we performed $10$ runs for each hyperparameter choice and chose the one with the highest total reward over these runs. We then measured the performance of the algorithms over $10^3$ runs. This ensures an unbiased evaluation.

\begin{table}[h]
  \centering
  \begin{tabular}{c|c|c}
    Methods & parameter & function\\
    \hline
    $Q(\lambda)$ & $\egreedy_0$ & exploration\\
    UCRL & $\delta$ & confidence interval\\
    MCBRL, U-MCBRL & $\nsam$ & number of samples\\
    BGBRL, $Q(\lambda)$ & $\step_0$ & step size
  \end{tabular}
  \caption{Automatically tuned hyperparameters}
  \label{tab:hyperparameters}
\end{table}
The set of hyper-parameters that were automatically tuned for each method are listed in Table~\ref{tab:hyperparameters}. For $Q(\lambda)$, we fixed $\lambda = 0.9$ and used an $\egreedy$-greedy strategy with a decaying rate and tuned initial value $\epsilon_0$. For UCRL, we tune the interval error probability $\delta$. Gradient algorithms require tuning the initial step-size parameter $\step_0$.  Monte-Carlo algorithms require tuning the number of samples $\nsam$. UCRL, MCBRL and U-MCBRL all used the same policy-switching heuristic.

\subsection{Domains}

We employed standard domains from discrete-state problems in exploration in reinforcement learning. Thus, Bayesian inference is closed-form, as we can use a Dirichlet-product prior for the transitions and a Normal-Gamma prior for the reward. Value function parametrisation is tabular, i.e. there is one parameter per state-action pair. These domains are the Chain problem~\cite{strens2000bayesian}, River-Swim~\cite{strehl2008analysis}, Double-Loop~\cite{strens2000bayesian}. In addition, we consider the mountain car domain of~\cite{Sutton+Barto:1998}, using a uniform $5 \times 5$ grid as features. All domains employed a discount factor $\disc = 0.99$.
\begin{table}[ht]
  \centering
  \begin{tabular}{c|c|c|c||c}
    \multicolumn{5}{c}{Chain} \\ \hline
    $Q(\lambda)$ & 1993.9 & 1999.7 & 2005.4     & 3 \\ 
    UCRL         & 3543.5 & 3547.5 & 3551.3     & 1613 \\
    \textbf{MCBRL} & 3610.5 & 3616.1 & 3621.7   & 464\\
    \textbf{U-MCBRL} & 3617.8 & 3623.4 & 3629.1 & 1560\\
    BGBRL        & 3593.6 & 3598.3 & 3602.7     & 48\\
    \hline
    \multicolumn{5}{c}{Double Loop} \\ \hline
    $Q(\lambda)$ & 2053.7 & 2058.1 & 2062.1     &   5 \\
    UCRL         & 3841.0 & 3841.0 & 3841.0     & 369 \\ 
    \textbf{MCBRL} & 3949.5 & 3950.2 & 3951.0   & 2343 \\ 
    U-MCBRL        & 3946.7 & 3947.5 & 3948.3   & 5135 \\ 
    BGBRL        & 3925.3 & 3926.2 & 3927.0    & 96     \\ 
    \hline
    \multicolumn{5}{c}{River Swim} \\ \hline
    $Q(\lambda)$ & 5.0 & 5.0 & 5.0           & 5    \\
    UCRL         & 312.4 & 313.8 & 315.3     & 240  \\ 
    \textbf{MCBRL} & 624.0 & 625.4 & 626.8   & 1187 \\ 
    \textbf{U-MCBRL} & 626.3 & 627.6 & 629.0 & 2329 \\ 
    BGBRL        & 600.3 & 601.7 & 603.2     & 69   \\ 
    \hline
    \multicolumn{5}{c}{Mountain Car $5 \times 5$} \\ \hline
    $Q(\lambda)$ & -9957.6 & -9957.0 & -9956.3     & 15 \\
    UCRL         & -9952.9 & -9951.6 & -9950.3     & 1908 \\ 
    MCBRL         & -9829.1 & -9827.2 & -9825.5    & 35733 \\ 
    \textbf{U-MCBRL} & -9811.8 & -9810.2 & -9808.6 & 66252\\ 
    BGBRL        & -9883.2 & -9881.9 & -9880.6     & 886 \\\hline
    Method & 95\% lower & mean & 95\% upper & CPU (s)
  \end{tabular}
  \caption{Total reward and CPU time}
  \label{tab:total-reward}
\end{table}

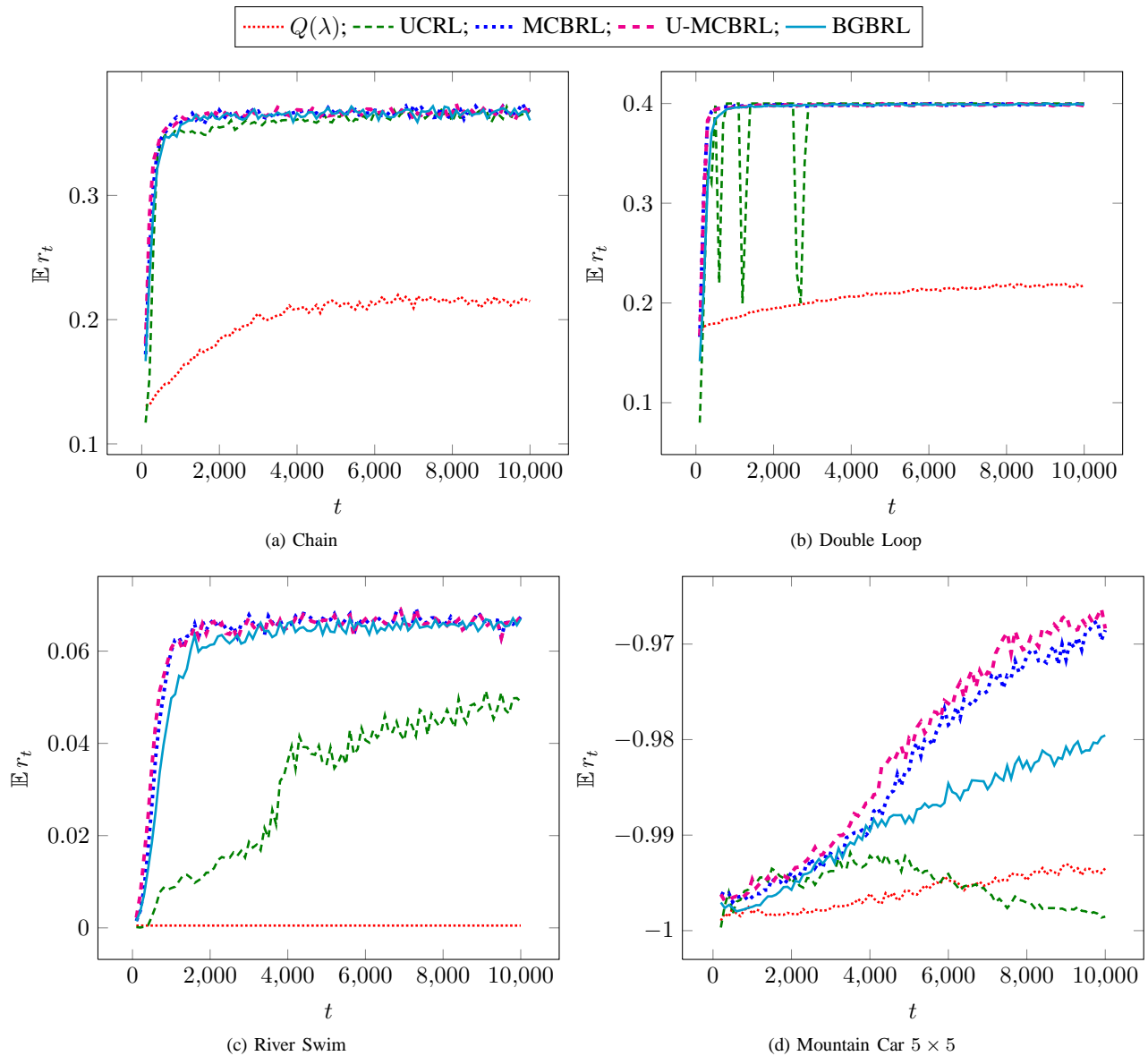
\begin{figure*}[ht]
  \centering
  \ref{leg:online}
  \subfloat[Chain]{\input{figures/Chain_payoff.tex}}
  \subfloat[Double Loop]{\input{figures/DoubleLoop_payoff.tex}}
  \\
  \subfloat[River Swim]{\input{figures/RiverSwim_payoff.tex}}
  \subfloat[Mountain Car $5 \times 5$]{\input{figures/MountainCar_payoff.tex}}
  \caption{Reward per step, smoothed over 100 steps and averaged over $10^3$ runs.}
  \label{fig:reward}
\end{figure*}
\subsection{Results}
From the online performance results shown in Fig~\ref{fig:reward}, it is clear that apart from $Q(\lambda)$, all algorithms are performing relatively similarly in the simpler environments. However, UCRL converges somewhat more slowly and is particularly unstable in the Mountain Car domain.~\footnote{Due to the discretisation, this domain is no longer fully observable.}

A clearer view of the performance of each algorithm is can be seen in Table~\ref{tab:total-reward}, in terms of the average total reward obtained.
 It additionally shows  the $95\%$ lower and upper confidence bound calculated on the mean (shown in the middle column) via $10^4$ bootstrap samples.
The best-performing methods in each environment (taking into account the bootstrap intervals) are shown in bold.
One immediately notices that MCBRL and U-MCBRL are usually tied for best. This is perhaps not surprising, as they have the same structure: in fact, for $\nsam=1$, they are equivalent to Thompson sampling~\cite{strens2000bayesian}, as mentioned in~\cite{dimitrakakis:mmbi:ewrl:2011}. However, MCBRL uses a \emph{lower bound} on the value function, while U-MCBRL an \emph{upper bound}, which makes it more optimistic.\footnote{Although we did not explicitly consider Thompson sampling, we note that the hyperparameter $\nsam = 1$ corresponding to Thompson sampling was never chosen by the automatic procedure as it always had worse performance than taking more samples. Nevertheless, its performance over the $10^3$ runs was always significantly worse than those of MCBRL and U-MCBRL.}

The most significant finding, however is that BGBRL is a relatively close second most of the time, performing better than all the remaining algorithms. This is despite its computational simplicity.

\section{Conclusion}

This paper introduced a set of Monte-Carlo algorithms for Bayesian reinforcement learning. The first, U-MCBRL is a modification of a lower-bound algorithm to an upper-bound setting, which has very good performance but has relatively high computational complexity. The second, DGBRL, is a type of gradient-based algorithm for approximating either the lower or the upper bound, but nevertheless does not necessarily alleviate the problem of complexity. Finally, BGBRL defines a novel type of Bellman error minimisation, on the Bayes-expected value function. By performing gradient descent to reduce this error through sampling possible MDPs, we obtain an efficient and highly competitive algorithm.

The algorithms were tested using an unbiased experimental methodology, whereby hyperparameters were automatically selected from a small number of runs. This ensures that algorithmic brittleness is not an issue. In all of those experiments, U-MCBRL and its sibling, MCBRL outperformed all alternatives. However, BGBRL was a close runner-up, even though it is computationally much simpler, as it does not require performing value iteration.

A subject that this paper has not touched upon is the theoretical performance of U-MCBRL and BGBRL. For the first, the results for MCBRL~\cite{dimitrakakis:mmbi:ewrl:2011} should be applicable with few modifications. The performance analysis of BGBRL-like algorithms, on the other hand, is a completely open question at the moment.

\bibliographystyle{IEEEtran}
\bibliography{../../bib/misc,../../bib/mine,../../bib/my_citations}

\end{document}

%% file: preamble.tex
\usepackage{enumerate}
\usepackage{amsmath}
\usepackage{amsfonts}
\usepackage{amssymb}
\usepackage{amsbsy}
\usepackage{dsfont}
\usepackage{theorem}
\usepackage{algorithm}
\usepackage{algorithmic}
\usepackage{mathrsfs}
\usepackage{paralist}
\usepackage{epsfig}
\usepackage{subfig}
\usepackage{makeidx} 
\usepackage{color}
\usepackage{pstricks}
\usepackage{pst-node}
\usepackage{multido} 
\usepackage{varioref}
\usepackage{wrapfig} 

\theoremstyle{plain} \newtheorem{corollary}{Corollary}

\newcommand \E {\mathop{\mbox{\ensuremath{\mathbb{E}}}}\nolimits}

\renewcommand \Pr {\mathop{\mbox{\ensuremath{\mathbb{P}}}}\nolimits}

\newcommand{\cset}[2]{\left\{\, #1 ~\middle|~ #2 \,\right\} }

\newcommand\Reals {{\mathds{R}}}

\newcommand \CA {{\mathcal{A}}}

\newcommand \CH {{\mathcal{H}}}

\newcommand \CM {{\mathcal{M}}}

\newcommand \CS {{\mathcal{S}}}

\newcommand \defn {\mathrel{\triangleq}}

\newcommand \argmax{\mathop{\rm arg\,max}}

\newcommand \norm[1]{\left\|#1\right\|}

\DeclareMathAlphabet{\mathpzc}{OT1}{pzc}{m}{it}

\newcommand \bel {\xi}

\newcommand \pol {\pi}
\newcommand \Pols {\Pi}
\newcommand \mdp {\mu}
\newcommand \MDPs {\CM}

\newcommand \Emp {\E_\mdp^\pol}
\newcommand \Pmp {\Pr_\mdp^\pol}
\newcommand \Ebp {\E_\bel^\pol}

\newcommand\dd{\,\mathrm{d}}

\newcommand \seq[2]{#1^{#2}}
\newcommand \pseq[3]{#1_{#2}^{#3}}

\newcommand \param {{\theta}}
\newcommand \Params {\Theta}

\newcommand \egreedy {\varepsilon}
\newcommand \nsam {\ensuremath{N}}

\newcommand \rew {\rho}
\newcommand \disc {\gamma}

\newcommand \step {\eta}

\newcommand \vv {v}
\newcommand \vq {q}

\newcommand \hist {h}

\newcommand \Urange {U_{\max}}

\def\clap#1{\hbox to 0pt{\hss#1\hss}}

\def\mathrlap{\mathpalette\mathrlapinternal}

\def\mathrlapinternal#1#2{%
           \rlap{$\mathsurround=0pt#1{#2}$}}

%% file: figures/Chain_payoff.tex
\begin{tikzpicture}
\begin{axis}[
xlabel={$t$},
ylabel={$\E r_t$},
scaled x ticks = false,
x tick label style = {/pgf/number format/fixed},
scaled y ticks = false,
y tick label style = {/pgf/number format/fixed},
legend columns=-1,
legend entries={$Q(\lambda)$;, UCRL;, MCBRL;, U-MCBRL;, BGBRL},
legend to name=leg:online
]
\addplot[line width=1, densely dotted,color=red,] coordinates {
(100.0000000000, 0.1306460000)
(200.0000000000, 0.1312960000)
(300.0000000000, 0.1365360000)
(400.0000000000, 0.1412900000)
(500.0000000000, 0.1446280000)
(600.0000000000, 0.1482960000)
(700.0000000000, 0.1489180000)
(800.0000000000, 0.1530660000)
(900.0000000000, 0.1556240000)
(1000.0000000000, 0.1600620000)
(1100.0000000000, 0.1641480000)
(1200.0000000000, 0.1645920000)
(1300.0000000000, 0.1695740000)
(1400.0000000000, 0.1695640000)
(1500.0000000000, 0.1756100000)
(1600.0000000000, 0.1739540000)
(1700.0000000000, 0.1759820000)
(1800.0000000000, 0.1774080000)
(1900.0000000000, 0.1791280000)
(2000.0000000000, 0.1843880000)
(2100.0000000000, 0.1860860000)
(2200.0000000000, 0.1880060000)
(2300.0000000000, 0.1871280000)
(2400.0000000000, 0.1923380000)
(2500.0000000000, 0.1912740000)
(2600.0000000000, 0.1945680000)
(2700.0000000000, 0.1953760000)
(2800.0000000000, 0.1961960000)
(2900.0000000000, 0.2018720000)
(3000.0000000000, 0.2048980000)
(3100.0000000000, 0.2017500000)
(3200.0000000000, 0.1993600000)
(3300.0000000000, 0.2005320000)
(3400.0000000000, 0.2019380000)
(3500.0000000000, 0.2036540000)
(3600.0000000000, 0.2039500000)
(3700.0000000000, 0.2070860000)
(3800.0000000000, 0.2112960000)
(3900.0000000000, 0.2082100000)
(4000.0000000000, 0.2084120000)
(4100.0000000000, 0.2094880000)
(4200.0000000000, 0.2101620000)
(4300.0000000000, 0.2055640000)
(4400.0000000000, 0.2085300000)
(4500.0000000000, 0.2118120000)
(4600.0000000000, 0.2052860000)
(4700.0000000000, 0.2118220000)
(4800.0000000000, 0.2139760000)
(4900.0000000000, 0.2153740000)
(5000.0000000000, 0.2134940000)
(5100.0000000000, 0.2109040000)
(5200.0000000000, 0.2126760000)
(5300.0000000000, 0.2137140000)
(5400.0000000000, 0.2099860000)
(5500.0000000000, 0.2081300000)
(5600.0000000000, 0.2141240000)
(5700.0000000000, 0.2100480000)
(5800.0000000000, 0.2083160000)
(5900.0000000000, 0.2144060000)
(6000.0000000000, 0.2123520000)
(6100.0000000000, 0.2136660000)
(6200.0000000000, 0.2187220000)
(6300.0000000000, 0.2137020000)
(6400.0000000000, 0.2133780000)
(6500.0000000000, 0.2152300000)
(6600.0000000000, 0.2198720000)
(6700.0000000000, 0.2144140000)
(6800.0000000000, 0.2183480000)
(6900.0000000000, 0.2157000000)
(7000.0000000000, 0.2093420000)
(7100.0000000000, 0.2188220000)
(7200.0000000000, 0.2164220000)
(7300.0000000000, 0.2149500000)
(7400.0000000000, 0.2178900000)
(7500.0000000000, 0.2174480000)
(7600.0000000000, 0.2149680000)
(7700.0000000000, 0.2144360000)
(7800.0000000000, 0.2171020000)
(7900.0000000000, 0.2136940000)
(8000.0000000000, 0.2145820000)
(8100.0000000000, 0.2152780000)
(8200.0000000000, 0.2184680000)
(8300.0000000000, 0.2166840000)
(8400.0000000000, 0.2126600000)
(8500.0000000000, 0.2150200000)
(8600.0000000000, 0.2158020000)
(8700.0000000000, 0.2099160000)
(8800.0000000000, 0.2115700000)
(8900.0000000000, 0.2161480000)
(9000.0000000000, 0.2136540000)
(9100.0000000000, 0.2183720000)
(9200.0000000000, 0.2138740000)
(9300.0000000000, 0.2159700000)
(9400.0000000000, 0.2150540000)
(9500.0000000000, 0.2114000000)
(9600.0000000000, 0.2148040000)
(9700.0000000000, 0.2136180000)
(9800.0000000000, 0.2167880000)
(9900.0000000000, 0.2160720000)
(10000.0000000000, 0.2148100000)
};
\addplot[line width=1, densely dashed,color=green!50!black,] coordinates {
(100.0000000000, 0.1170860000)
(200.0000000000, 0.1537200000)
(300.0000000000, 0.2590860000)
(400.0000000000, 0.3299280000)
(500.0000000000, 0.3432540000)
(600.0000000000, 0.3485080000)
(700.0000000000, 0.3494340000)
(800.0000000000, 0.3489800000)
(900.0000000000, 0.3543000000)
(1000.0000000000, 0.3522660000)
(1100.0000000000, 0.3507220000)
(1200.0000000000, 0.3503440000)
(1300.0000000000, 0.3516920000)
(1400.0000000000, 0.3496420000)
(1500.0000000000, 0.3508880000)
(1600.0000000000, 0.3489000000)
(1700.0000000000, 0.3483640000)
(1800.0000000000, 0.3546060000)
(1900.0000000000, 0.3561220000)
(2000.0000000000, 0.3546160000)
(2100.0000000000, 0.3553020000)
(2200.0000000000, 0.3562820000)
(2300.0000000000, 0.3605120000)
(2400.0000000000, 0.3527660000)
(2500.0000000000, 0.3550940000)
(2600.0000000000, 0.3579960000)
(2700.0000000000, 0.3607640000)
(2800.0000000000, 0.3578260000)
(2900.0000000000, 0.3590180000)
(3000.0000000000, 0.3569200000)
(3100.0000000000, 0.3564340000)
(3200.0000000000, 0.3597460000)
(3300.0000000000, 0.3591680000)
(3400.0000000000, 0.3601820000)
(3500.0000000000, 0.3575660000)
(3600.0000000000, 0.3604160000)
(3700.0000000000, 0.3606700000)
(3800.0000000000, 0.3605860000)
(3900.0000000000, 0.3605800000)
(4000.0000000000, 0.3609580000)
(4100.0000000000, 0.3625280000)
(4200.0000000000, 0.3597620000)
(4300.0000000000, 0.3606140000)
(4400.0000000000, 0.3595980000)
(4500.0000000000, 0.3584900000)
(4600.0000000000, 0.3607720000)
(4700.0000000000, 0.3648860000)
(4800.0000000000, 0.3594920000)
(4900.0000000000, 0.3587440000)
(5000.0000000000, 0.3637380000)
(5100.0000000000, 0.3632380000)
(5200.0000000000, 0.3669160000)
(5300.0000000000, 0.3617940000)
(5400.0000000000, 0.3580680000)
(5500.0000000000, 0.3614880000)
(5600.0000000000, 0.3606000000)
(5700.0000000000, 0.3641180000)
(5800.0000000000, 0.3617060000)
(5900.0000000000, 0.3597760000)
(6000.0000000000, 0.3670700000)
(6100.0000000000, 0.3651260000)
(6200.0000000000, 0.3660640000)
(6300.0000000000, 0.3625940000)
(6400.0000000000, 0.3610300000)
(6500.0000000000, 0.3685760000)
(6600.0000000000, 0.3626480000)
(6700.0000000000, 0.3666140000)
(6800.0000000000, 0.3652100000)
(6900.0000000000, 0.3661320000)
(7000.0000000000, 0.3690760000)
(7100.0000000000, 0.3650300000)
(7200.0000000000, 0.3669780000)
(7300.0000000000, 0.3621020000)
(7400.0000000000, 0.3649100000)
(7500.0000000000, 0.3645620000)
(7600.0000000000, 0.3643240000)
(7700.0000000000, 0.3701380000)
(7800.0000000000, 0.3637740000)
(7900.0000000000, 0.3654780000)
(8000.0000000000, 0.3594360000)
(8100.0000000000, 0.3669820000)
(8200.0000000000, 0.3625480000)
(8300.0000000000, 0.3627900000)
(8400.0000000000, 0.3643240000)
(8500.0000000000, 0.3684060000)
(8600.0000000000, 0.3657300000)
(8700.0000000000, 0.3680640000)
(8800.0000000000, 0.3653880000)
(8900.0000000000, 0.3663600000)
(9000.0000000000, 0.3586280000)
(9100.0000000000, 0.3647380000)
(9200.0000000000, 0.3649020000)
(9300.0000000000, 0.3617020000)
(9400.0000000000, 0.3712760000)
(9500.0000000000, 0.3636600000)
(9600.0000000000, 0.3622220000)
(9700.0000000000, 0.3638760000)
(9800.0000000000, 0.3653840000)
(9900.0000000000, 0.3644340000)
(10000.0000000000, 0.3687500000)
};
\addplot[line width=1.5, dotted,color=blue,] coordinates {
(100.0000000000, 0.1720640000)
(200.0000000000, 0.2528960000)
(300.0000000000, 0.3075960000)
(400.0000000000, 0.3351580000)
(500.0000000000, 0.3458300000)
(600.0000000000, 0.3504580000)
(700.0000000000, 0.3547120000)
(800.0000000000, 0.3601080000)
(900.0000000000, 0.3629580000)
(1000.0000000000, 0.3588140000)
(1100.0000000000, 0.3586760000)
(1200.0000000000, 0.3614720000)
(1300.0000000000, 0.3659960000)
(1400.0000000000, 0.3608920000)
(1500.0000000000, 0.3644180000)
(1600.0000000000, 0.3571680000)
(1700.0000000000, 0.3630440000)
(1800.0000000000, 0.3634660000)
(1900.0000000000, 0.3611840000)
(2000.0000000000, 0.3684620000)
(2100.0000000000, 0.3640240000)
(2200.0000000000, 0.3615220000)
(2300.0000000000, 0.3657640000)
(2400.0000000000, 0.3661940000)
(2500.0000000000, 0.3609880000)
(2600.0000000000, 0.3630080000)
(2700.0000000000, 0.3657600000)
(2800.0000000000, 0.3708060000)
(2900.0000000000, 0.3667320000)
(3000.0000000000, 0.3678760000)
(3100.0000000000, 0.3634980000)
(3200.0000000000, 0.3622780000)
(3300.0000000000, 0.3626040000)
(3400.0000000000, 0.3667860000)
(3500.0000000000, 0.3653360000)
(3600.0000000000, 0.3667020000)
(3700.0000000000, 0.3649280000)
(3800.0000000000, 0.3653360000)
(3900.0000000000, 0.3664980000)
(4000.0000000000, 0.3674860000)
(4100.0000000000, 0.3675840000)
(4200.0000000000, 0.3689580000)
(4300.0000000000, 0.3676980000)
(4400.0000000000, 0.3654100000)
(4500.0000000000, 0.3664160000)
(4600.0000000000, 0.3638820000)
(4700.0000000000, 0.3653180000)
(4800.0000000000, 0.3705560000)
(4900.0000000000, 0.3665220000)
(5000.0000000000, 0.3699940000)
(5100.0000000000, 0.3640040000)
(5200.0000000000, 0.3690260000)
(5300.0000000000, 0.3657780000)
(5400.0000000000, 0.3664020000)
(5500.0000000000, 0.3688940000)
(5600.0000000000, 0.3647060000)
(5700.0000000000, 0.3657620000)
(5800.0000000000, 0.3663500000)
(5900.0000000000, 0.3663100000)
(6000.0000000000, 0.3685440000)
(6100.0000000000, 0.3681520000)
(6200.0000000000, 0.3642980000)
(6300.0000000000, 0.3642480000)
(6400.0000000000, 0.3676260000)
(6500.0000000000, 0.3652380000)
(6600.0000000000, 0.3650460000)
(6700.0000000000, 0.3689140000)
(6800.0000000000, 0.3718840000)
(6900.0000000000, 0.3641820000)
(7000.0000000000, 0.3725260000)
(7100.0000000000, 0.3677340000)
(7200.0000000000, 0.3707700000)
(7300.0000000000, 0.3673620000)
(7400.0000000000, 0.3645460000)
(7500.0000000000, 0.3627040000)
(7600.0000000000, 0.3630080000)
(7700.0000000000, 0.3705280000)
(7800.0000000000, 0.3649740000)
(7900.0000000000, 0.3632220000)
(8000.0000000000, 0.3670000000)
(8100.0000000000, 0.3703180000)
(8200.0000000000, 0.3696360000)
(8300.0000000000, 0.3614940000)
(8400.0000000000, 0.3675800000)
(8500.0000000000, 0.3651400000)
(8600.0000000000, 0.3716560000)
(8700.0000000000, 0.3674640000)
(8800.0000000000, 0.3683560000)
(8900.0000000000, 0.3687000000)
(9000.0000000000, 0.3678700000)
(9100.0000000000, 0.3648980000)
(9200.0000000000, 0.3674420000)
(9300.0000000000, 0.3716260000)
(9400.0000000000, 0.3664400000)
(9500.0000000000, 0.3648340000)
(9600.0000000000, 0.3661480000)
(9700.0000000000, 0.3686020000)
(9800.0000000000, 0.3739880000)
(9900.0000000000, 0.3684320000)
(10000.0000000000, 0.3689700000)
};
\addplot[line width=1.5, dashed,color=magenta,] coordinates {
(100.0000000000, 0.1807080000)
(200.0000000000, 0.2848060000)
(300.0000000000, 0.3276920000)
(400.0000000000, 0.3384000000)
(500.0000000000, 0.3493440000)
(600.0000000000, 0.3528900000)
(700.0000000000, 0.3568780000)
(800.0000000000, 0.3556980000)
(900.0000000000, 0.3583940000)
(1000.0000000000, 0.3640180000)
(1100.0000000000, 0.3615080000)
(1200.0000000000, 0.3611660000)
(1300.0000000000, 0.3637400000)
(1400.0000000000, 0.3624640000)
(1500.0000000000, 0.3665660000)
(1600.0000000000, 0.3652940000)
(1700.0000000000, 0.3662780000)
(1800.0000000000, 0.3613840000)
(1900.0000000000, 0.3665860000)
(2000.0000000000, 0.3651140000)
(2100.0000000000, 0.3678220000)
(2200.0000000000, 0.3648960000)
(2300.0000000000, 0.3661700000)
(2400.0000000000, 0.3681600000)
(2500.0000000000, 0.3613740000)
(2600.0000000000, 0.3645460000)
(2700.0000000000, 0.3634720000)
(2800.0000000000, 0.3699360000)
(2900.0000000000, 0.3661260000)
(3000.0000000000, 0.3647380000)
(3100.0000000000, 0.3648860000)
(3200.0000000000, 0.3635640000)
(3300.0000000000, 0.3659080000)
(3400.0000000000, 0.3651880000)
(3500.0000000000, 0.3708760000)
(3600.0000000000, 0.3688320000)
(3700.0000000000, 0.3659740000)
(3800.0000000000, 0.3641100000)
(3900.0000000000, 0.3662560000)
(4000.0000000000, 0.3689400000)
(4100.0000000000, 0.3687060000)
(4200.0000000000, 0.3682880000)
(4300.0000000000, 0.3663040000)
(4400.0000000000, 0.3700340000)
(4500.0000000000, 0.3681600000)
(4600.0000000000, 0.3636160000)
(4700.0000000000, 0.3671740000)
(4800.0000000000, 0.3640080000)
(4900.0000000000, 0.3645520000)
(5000.0000000000, 0.3651400000)
(5100.0000000000, 0.3688080000)
(5200.0000000000, 0.3631540000)
(5300.0000000000, 0.3666760000)
(5400.0000000000, 0.3684940000)
(5500.0000000000, 0.3713740000)
(5600.0000000000, 0.3665060000)
(5700.0000000000, 0.3639500000)
(5800.0000000000, 0.3694840000)
(5900.0000000000, 0.3693940000)
(6000.0000000000, 0.3656060000)
(6100.0000000000, 0.3666580000)
(6200.0000000000, 0.3672300000)
(6300.0000000000, 0.3679700000)
(6400.0000000000, 0.3708260000)
(6500.0000000000, 0.3657900000)
(6600.0000000000, 0.3647240000)
(6700.0000000000, 0.3686280000)
(6800.0000000000, 0.3669480000)
(6900.0000000000, 0.3684900000)
(7000.0000000000, 0.3609160000)
(7100.0000000000, 0.3662960000)
(7200.0000000000, 0.3645780000)
(7300.0000000000, 0.3683760000)
(7400.0000000000, 0.3656960000)
(7500.0000000000, 0.3674040000)
(7600.0000000000, 0.3679640000)
(7700.0000000000, 0.3641480000)
(7800.0000000000, 0.3704560000)
(7900.0000000000, 0.3611100000)
(8000.0000000000, 0.3664380000)
(8100.0000000000, 0.3739720000)
(8200.0000000000, 0.3632340000)
(8300.0000000000, 0.3674960000)
(8400.0000000000, 0.3646940000)
(8500.0000000000, 0.3682180000)
(8600.0000000000, 0.3652340000)
(8700.0000000000, 0.3657760000)
(8800.0000000000, 0.3699180000)
(8900.0000000000, 0.3666480000)
(9000.0000000000, 0.3680280000)
(9100.0000000000, 0.3655560000)
(9200.0000000000, 0.3629560000)
(9300.0000000000, 0.3711860000)
(9400.0000000000, 0.3666120000)
(9500.0000000000, 0.3660460000)
(9600.0000000000, 0.3597600000)
(9700.0000000000, 0.3665360000)
(9800.0000000000, 0.3655540000)
(9900.0000000000, 0.3692680000)
(10000.0000000000, 0.3661520000)
};
\addplot[line width=1, solid,color=cyan!80!black,] coordinates {
(100.0000000000, 0.1666060000)
(200.0000000000, 0.2338140000)
(300.0000000000, 0.2886040000)
(400.0000000000, 0.3218280000)
(500.0000000000, 0.3347460000)
(600.0000000000, 0.3474340000)
(700.0000000000, 0.3464120000)
(800.0000000000, 0.3498500000)
(900.0000000000, 0.3473840000)
(1000.0000000000, 0.3565780000)
(1100.0000000000, 0.3581820000)
(1200.0000000000, 0.3601400000)
(1300.0000000000, 0.3603020000)
(1400.0000000000, 0.3649400000)
(1500.0000000000, 0.3607120000)
(1600.0000000000, 0.3637480000)
(1700.0000000000, 0.3615500000)
(1800.0000000000, 0.3661800000)
(1900.0000000000, 0.3621180000)
(2000.0000000000, 0.3612860000)
(2100.0000000000, 0.3645100000)
(2200.0000000000, 0.3663920000)
(2300.0000000000, 0.3639560000)
(2400.0000000000, 0.3654220000)
(2500.0000000000, 0.3621100000)
(2600.0000000000, 0.3609200000)
(2700.0000000000, 0.3623400000)
(2800.0000000000, 0.3665240000)
(2900.0000000000, 0.3692360000)
(3000.0000000000, 0.3617160000)
(3100.0000000000, 0.3648180000)
(3200.0000000000, 0.3633660000)
(3300.0000000000, 0.3652120000)
(3400.0000000000, 0.3640840000)
(3500.0000000000, 0.3667700000)
(3600.0000000000, 0.3650200000)
(3700.0000000000, 0.3657900000)
(3800.0000000000, 0.3672600000)
(3900.0000000000, 0.3591160000)
(4000.0000000000, 0.3661780000)
(4100.0000000000, 0.3630500000)
(4200.0000000000, 0.3635860000)
(4300.0000000000, 0.3670100000)
(4400.0000000000, 0.3670680000)
(4500.0000000000, 0.3618380000)
(4600.0000000000, 0.3657920000)
(4700.0000000000, 0.3669760000)
(4800.0000000000, 0.3706540000)
(4900.0000000000, 0.3656860000)
(5000.0000000000, 0.3659620000)
(5100.0000000000, 0.3685920000)
(5200.0000000000, 0.3654040000)
(5300.0000000000, 0.3657900000)
(5400.0000000000, 0.3716060000)
(5500.0000000000, 0.3680840000)
(5600.0000000000, 0.3663560000)
(5700.0000000000, 0.3688900000)
(5800.0000000000, 0.3628740000)
(5900.0000000000, 0.3652700000)
(6000.0000000000, 0.3683220000)
(6100.0000000000, 0.3606180000)
(6200.0000000000, 0.3676860000)
(6300.0000000000, 0.3636900000)
(6400.0000000000, 0.3627160000)
(6500.0000000000, 0.3632680000)
(6600.0000000000, 0.3665860000)
(6700.0000000000, 0.3672580000)
(6800.0000000000, 0.3691720000)
(6900.0000000000, 0.3690040000)
(7000.0000000000, 0.3646240000)
(7100.0000000000, 0.3590640000)
(7200.0000000000, 0.3677780000)
(7300.0000000000, 0.3673560000)
(7400.0000000000, 0.3679540000)
(7500.0000000000, 0.3718080000)
(7600.0000000000, 0.3639920000)
(7700.0000000000, 0.3617260000)
(7800.0000000000, 0.3714700000)
(7900.0000000000, 0.3705800000)
(8000.0000000000, 0.3624440000)
(8100.0000000000, 0.3684160000)
(8200.0000000000, 0.3642940000)
(8300.0000000000, 0.3667340000)
(8400.0000000000, 0.3687060000)
(8500.0000000000, 0.3645540000)
(8600.0000000000, 0.3656660000)
(8700.0000000000, 0.3658020000)
(8800.0000000000, 0.3692600000)
(8900.0000000000, 0.3619980000)
(9000.0000000000, 0.3660320000)
(9100.0000000000, 0.3677360000)
(9200.0000000000, 0.3688780000)
(9300.0000000000, 0.3702360000)
(9400.0000000000, 0.3685240000)
(9500.0000000000, 0.3649340000)
(9600.0000000000, 0.3696760000)
(9700.0000000000, 0.3649660000)
(9800.0000000000, 0.3675860000)
(9900.0000000000, 0.3674580000)
(10000.0000000000, 0.3603840000)
};
\end{axis}\end{tikzpicture}

%% file: figures/DoubleLoop_payoff.tex
\begin{tikzpicture}
\begin{axis}[
xlabel={$t$},
ylabel={$\E r_t$},
scaled x ticks = false,
x tick label style = {/pgf/number format/fixed},
scaled y ticks = false,
y tick label style = {/pgf/number format/fixed},
]
\addplot[line width=1, densely dotted,color=red,] coordinates {
(100.0000000000, 0.1677600000)
(200.0000000000, 0.1761300000)
(300.0000000000, 0.1780700000)
(400.0000000000, 0.1791000000)
(500.0000000000, 0.1794000000)
(600.0000000000, 0.1799500000)
(700.0000000000, 0.1832600000)
(800.0000000000, 0.1836600000)
(900.0000000000, 0.1846800000)
(1000.0000000000, 0.1848000000)
(1100.0000000000, 0.1857300000)
(1200.0000000000, 0.1875200000)
(1300.0000000000, 0.1880300000)
(1400.0000000000, 0.1913200000)
(1500.0000000000, 0.1909600000)
(1600.0000000000, 0.1920000000)
(1700.0000000000, 0.1934400000)
(1800.0000000000, 0.1927000000)
(1900.0000000000, 0.1943600000)
(2000.0000000000, 0.1939100000)
(2100.0000000000, 0.1956900000)
(2200.0000000000, 0.1951800000)
(2300.0000000000, 0.1965700000)
(2400.0000000000, 0.1971100000)
(2500.0000000000, 0.1977000000)
(2600.0000000000, 0.1979000000)
(2700.0000000000, 0.1985300000)
(2800.0000000000, 0.1991800000)
(2900.0000000000, 0.2000200000)
(3000.0000000000, 0.2000200000)
(3100.0000000000, 0.2026500000)
(3200.0000000000, 0.2019100000)
(3300.0000000000, 0.2023300000)
(3400.0000000000, 0.2021200000)
(3500.0000000000, 0.2035900000)
(3600.0000000000, 0.2035800000)
(3700.0000000000, 0.2052900000)
(3800.0000000000, 0.2057500000)
(3900.0000000000, 0.2061700000)
(4000.0000000000, 0.2068900000)
(4100.0000000000, 0.2060500000)
(4200.0000000000, 0.2066900000)
(4300.0000000000, 0.2071500000)
(4400.0000000000, 0.2096100000)
(4500.0000000000, 0.2080400000)
(4600.0000000000, 0.2092900000)
(4700.0000000000, 0.2098100000)
(4800.0000000000, 0.2084900000)
(4900.0000000000, 0.2101200000)
(5000.0000000000, 0.2095300000)
(5100.0000000000, 0.2099000000)
(5200.0000000000, 0.2093400000)
(5300.0000000000, 0.2103600000)
(5400.0000000000, 0.2114200000)
(5500.0000000000, 0.2122600000)
(5600.0000000000, 0.2137000000)
(5700.0000000000, 0.2123500000)
(5800.0000000000, 0.2140700000)
(5900.0000000000, 0.2132800000)
(6000.0000000000, 0.2126600000)
(6100.0000000000, 0.2134700000)
(6200.0000000000, 0.2135700000)
(6300.0000000000, 0.2122400000)
(6400.0000000000, 0.2145400000)
(6500.0000000000, 0.2147200000)
(6600.0000000000, 0.2151500000)
(6700.0000000000, 0.2131000000)
(6800.0000000000, 0.2139900000)
(6900.0000000000, 0.2161300000)
(7000.0000000000, 0.2141100000)
(7100.0000000000, 0.2151300000)
(7200.0000000000, 0.2141500000)
(7300.0000000000, 0.2142900000)
(7400.0000000000, 0.2157700000)
(7500.0000000000, 0.2163600000)
(7600.0000000000, 0.2148500000)
(7700.0000000000, 0.2190900000)
(7800.0000000000, 0.2164500000)
(7900.0000000000, 0.2168800000)
(8000.0000000000, 0.2172800000)
(8100.0000000000, 0.2185800000)
(8200.0000000000, 0.2155800000)
(8300.0000000000, 0.2196100000)
(8400.0000000000, 0.2181000000)
(8500.0000000000, 0.2165300000)
(8600.0000000000, 0.2163300000)
(8700.0000000000, 0.2168300000)
(8800.0000000000, 0.2184200000)
(8900.0000000000, 0.2162800000)
(9000.0000000000, 0.2163800000)
(9100.0000000000, 0.2193800000)
(9200.0000000000, 0.2181700000)
(9300.0000000000, 0.2183000000)
(9400.0000000000, 0.2182800000)
(9500.0000000000, 0.2192000000)
(9600.0000000000, 0.2170500000)
(9700.0000000000, 0.2150100000)
(9800.0000000000, 0.2186300000)
(9900.0000000000, 0.2164800000)
(10000.0000000000, 0.2175700000)
};
\addplot[line width=1, densely dashed,color=green!50!black,] coordinates {
(100.0000000000, 0.0800000000)
(200.0000000000, 0.2000000000)
(300.0000000000, 0.3300000000)
(400.0000000000, 0.3200000000)
(500.0000000000, 0.4000000000)
(600.0000000000, 0.2200000000)
(700.0000000000, 0.3800000000)
(800.0000000000, 0.4000000000)
(900.0000000000, 0.4000000000)
(1000.0000000000, 0.4000000000)
(1100.0000000000, 0.4000000000)
(1200.0000000000, 0.2000000000)
(1300.0000000000, 0.3100000000)
(1400.0000000000, 0.4000000000)
(1500.0000000000, 0.4000000000)
(1600.0000000000, 0.4000000000)
(1700.0000000000, 0.4000000000)
(1800.0000000000, 0.4000000000)
(1900.0000000000, 0.4000000000)
(2000.0000000000, 0.4000000000)
(2100.0000000000, 0.4000000000)
(2200.0000000000, 0.4000000000)
(2300.0000000000, 0.4000000000)
(2400.0000000000, 0.4000000000)
(2500.0000000000, 0.4000000000)
(2600.0000000000, 0.2300000000)
(2700.0000000000, 0.2000000000)
(2800.0000000000, 0.3400000000)
(2900.0000000000, 0.4000000000)
(3000.0000000000, 0.4000000000)
(3100.0000000000, 0.4000000000)
(3200.0000000000, 0.4000000000)
(3300.0000000000, 0.4000000000)
(3400.0000000000, 0.4000000000)
(3500.0000000000, 0.4000000000)
(3600.0000000000, 0.4000000000)
(3700.0000000000, 0.4000000000)
(3800.0000000000, 0.4000000000)
(3900.0000000000, 0.4000000000)
(4000.0000000000, 0.4000000000)
(4100.0000000000, 0.4000000000)
(4200.0000000000, 0.4000000000)
(4300.0000000000, 0.4000000000)
(4400.0000000000, 0.4000000000)
(4500.0000000000, 0.4000000000)
(4600.0000000000, 0.4000000000)
(4700.0000000000, 0.4000000000)
(4800.0000000000, 0.4000000000)
(4900.0000000000, 0.4000000000)
(5000.0000000000, 0.4000000000)
(5100.0000000000, 0.4000000000)
(5200.0000000000, 0.4000000000)
(5300.0000000000, 0.4000000000)
(5400.0000000000, 0.4000000000)
(5500.0000000000, 0.4000000000)
(5600.0000000000, 0.4000000000)
(5700.0000000000, 0.4000000000)
(5800.0000000000, 0.4000000000)
(5900.0000000000, 0.4000000000)
(6000.0000000000, 0.4000000000)
(6100.0000000000, 0.4000000000)
(6200.0000000000, 0.4000000000)
(6300.0000000000, 0.4000000000)
(6400.0000000000, 0.4000000000)
(6500.0000000000, 0.4000000000)
(6600.0000000000, 0.4000000000)
(6700.0000000000, 0.4000000000)
(6800.0000000000, 0.4000000000)
(6900.0000000000, 0.4000000000)
(7000.0000000000, 0.4000000000)
(7100.0000000000, 0.4000000000)
(7200.0000000000, 0.4000000000)
(7300.0000000000, 0.4000000000)
(7400.0000000000, 0.4000000000)
(7500.0000000000, 0.4000000000)
(7600.0000000000, 0.4000000000)
(7700.0000000000, 0.4000000000)
(7800.0000000000, 0.4000000000)
(7900.0000000000, 0.4000000000)
(8000.0000000000, 0.4000000000)
(8100.0000000000, 0.4000000000)
(8200.0000000000, 0.4000000000)
(8300.0000000000, 0.4000000000)
(8400.0000000000, 0.4000000000)
(8500.0000000000, 0.4000000000)
(8600.0000000000, 0.4000000000)
(8700.0000000000, 0.4000000000)
(8800.0000000000, 0.4000000000)
(8900.0000000000, 0.4000000000)
(9000.0000000000, 0.4000000000)
(9100.0000000000, 0.4000000000)
(9200.0000000000, 0.4000000000)
(9300.0000000000, 0.4000000000)
(9400.0000000000, 0.4000000000)
(9500.0000000000, 0.4000000000)
(9600.0000000000, 0.4000000000)
(9700.0000000000, 0.4000000000)
(9800.0000000000, 0.4000000000)
(9900.0000000000, 0.4000000000)
(10000.0000000000, 0.4000000000)
};
\addplot[line width=1.5, dotted,color=blue,] coordinates {
(100.0000000000, 0.1660100000)
(200.0000000000, 0.3073400000)
(300.0000000000, 0.3810300000)
(400.0000000000, 0.3917900000)
(500.0000000000, 0.3947600000)
(600.0000000000, 0.3949300000)
(700.0000000000, 0.3966300000)
(800.0000000000, 0.3965400000)
(900.0000000000, 0.3957400000)
(1000.0000000000, 0.3979000000)
(1100.0000000000, 0.3970700000)
(1200.0000000000, 0.3976900000)
(1300.0000000000, 0.3972200000)
(1400.0000000000, 0.3979600000)
(1500.0000000000, 0.3982000000)
(1600.0000000000, 0.3979800000)
(1700.0000000000, 0.3973700000)
(1800.0000000000, 0.3983200000)
(1900.0000000000, 0.3978600000)
(2000.0000000000, 0.3986500000)
(2100.0000000000, 0.3985200000)
(2200.0000000000, 0.3986500000)
(2300.0000000000, 0.3988100000)
(2400.0000000000, 0.3978800000)
(2500.0000000000, 0.3977600000)
(2600.0000000000, 0.3987400000)
(2700.0000000000, 0.3983200000)
(2800.0000000000, 0.3973100000)
(2900.0000000000, 0.3975300000)
(3000.0000000000, 0.3976200000)
(3100.0000000000, 0.3979900000)
(3200.0000000000, 0.3984100000)
(3300.0000000000, 0.3985600000)
(3400.0000000000, 0.3983600000)
(3500.0000000000, 0.3982100000)
(3600.0000000000, 0.3976100000)
(3700.0000000000, 0.3990400000)
(3800.0000000000, 0.3990500000)
(3900.0000000000, 0.3981400000)
(4000.0000000000, 0.3981800000)
(4100.0000000000, 0.3984600000)
(4200.0000000000, 0.3986200000)
(4300.0000000000, 0.3987800000)
(4400.0000000000, 0.3986500000)
(4500.0000000000, 0.3977000000)
(4600.0000000000, 0.3979200000)
(4700.0000000000, 0.3990900000)
(4800.0000000000, 0.3986000000)
(4900.0000000000, 0.3991000000)
(5000.0000000000, 0.3988200000)
(5100.0000000000, 0.3986800000)
(5200.0000000000, 0.3990100000)
(5300.0000000000, 0.3990000000)
(5400.0000000000, 0.3995500000)
(5500.0000000000, 0.3998400000)
(5600.0000000000, 0.3992800000)
(5700.0000000000, 0.3988900000)
(5800.0000000000, 0.3993800000)
(5900.0000000000, 0.3987500000)
(6000.0000000000, 0.3996100000)
(6100.0000000000, 0.4000000000)
(6200.0000000000, 0.3990600000)
(6300.0000000000, 0.3987000000)
(6400.0000000000, 0.3981600000)
(6500.0000000000, 0.3983800000)
(6600.0000000000, 0.3986200000)
(6700.0000000000, 0.3987800000)
(6800.0000000000, 0.3990400000)
(6900.0000000000, 0.3980000000)
(7000.0000000000, 0.3993700000)
(7100.0000000000, 0.3998700000)
(7200.0000000000, 0.3993100000)
(7300.0000000000, 0.3986900000)
(7400.0000000000, 0.3987800000)
(7500.0000000000, 0.3996000000)
(7600.0000000000, 0.3992100000)
(7700.0000000000, 0.3975800000)
(7800.0000000000, 0.3984000000)
(7900.0000000000, 0.3995100000)
(8000.0000000000, 0.3986000000)
(8100.0000000000, 0.3996000000)
(8200.0000000000, 0.3995900000)
(8300.0000000000, 0.3988200000)
(8400.0000000000, 0.3979500000)
(8500.0000000000, 0.3988000000)
(8600.0000000000, 0.3991700000)
(8700.0000000000, 0.3993100000)
(8800.0000000000, 0.3993200000)
(8900.0000000000, 0.3990000000)
(9000.0000000000, 0.3991700000)
(9100.0000000000, 0.3989000000)
(9200.0000000000, 0.3986800000)
(9300.0000000000, 0.3990000000)
(9400.0000000000, 0.3986800000)
(9500.0000000000, 0.3987600000)
(9600.0000000000, 0.3990400000)
(9700.0000000000, 0.3994000000)
(9800.0000000000, 0.3992500000)
(9900.0000000000, 0.3993200000)
(10000.0000000000, 0.3996000000)
};
\addplot[line width=1.5, dashed,color=magenta,] coordinates {
(100.0000000000, 0.1670300000)
(200.0000000000, 0.3023700000)
(300.0000000000, 0.3830700000)
(400.0000000000, 0.3904500000)
(500.0000000000, 0.3940600000)
(600.0000000000, 0.3950800000)
(700.0000000000, 0.3948100000)
(800.0000000000, 0.3952600000)
(900.0000000000, 0.3955600000)
(1000.0000000000, 0.3971200000)
(1100.0000000000, 0.3959100000)
(1200.0000000000, 0.3970100000)
(1300.0000000000, 0.3966500000)
(1400.0000000000, 0.3976200000)
(1500.0000000000, 0.3972700000)
(1600.0000000000, 0.3985400000)
(1700.0000000000, 0.3981300000)
(1800.0000000000, 0.3978800000)
(1900.0000000000, 0.3981000000)
(2000.0000000000, 0.3976600000)
(2100.0000000000, 0.3979600000)
(2200.0000000000, 0.3963600000)
(2300.0000000000, 0.3987500000)
(2400.0000000000, 0.3982000000)
(2500.0000000000, 0.3983800000)
(2600.0000000000, 0.3975100000)
(2700.0000000000, 0.3985800000)
(2800.0000000000, 0.3982600000)
(2900.0000000000, 0.3977900000)
(3000.0000000000, 0.3986900000)
(3100.0000000000, 0.3982700000)
(3200.0000000000, 0.3980500000)
(3300.0000000000, 0.3981600000)
(3400.0000000000, 0.3986800000)
(3500.0000000000, 0.3980100000)
(3600.0000000000, 0.3977400000)
(3700.0000000000, 0.3978600000)
(3800.0000000000, 0.3983400000)
(3900.0000000000, 0.3990000000)
(4000.0000000000, 0.3983300000)
(4100.0000000000, 0.3984000000)
(4200.0000000000, 0.3978400000)
(4300.0000000000, 0.3976800000)
(4400.0000000000, 0.3976000000)
(4500.0000000000, 0.3986100000)
(4600.0000000000, 0.3979200000)
(4700.0000000000, 0.3981800000)
(4800.0000000000, 0.3998200000)
(4900.0000000000, 0.3986600000)
(5000.0000000000, 0.3979800000)
(5100.0000000000, 0.3985000000)
(5200.0000000000, 0.3983000000)
(5300.0000000000, 0.3979800000)
(5400.0000000000, 0.3981200000)
(5500.0000000000, 0.3981200000)
(5600.0000000000, 0.3978200000)
(5700.0000000000, 0.3978100000)
(5800.0000000000, 0.3987100000)
(5900.0000000000, 0.3992000000)
(6000.0000000000, 0.3992200000)
(6100.0000000000, 0.3994000000)
(6200.0000000000, 0.3992000000)
(6300.0000000000, 0.3982000000)
(6400.0000000000, 0.3978900000)
(6500.0000000000, 0.3987600000)
(6600.0000000000, 0.3993000000)
(6700.0000000000, 0.3992700000)
(6800.0000000000, 0.3993100000)
(6900.0000000000, 0.3988000000)
(7000.0000000000, 0.3988500000)
(7100.0000000000, 0.3989900000)
(7200.0000000000, 0.3987700000)
(7300.0000000000, 0.3985300000)
(7400.0000000000, 0.3985400000)
(7500.0000000000, 0.3990000000)
(7600.0000000000, 0.3997800000)
(7700.0000000000, 0.3987600000)
(7800.0000000000, 0.3987300000)
(7900.0000000000, 0.3990400000)
(8000.0000000000, 0.3992000000)
(8100.0000000000, 0.3998100000)
(8200.0000000000, 0.3989600000)
(8300.0000000000, 0.3989600000)
(8400.0000000000, 0.3994300000)
(8500.0000000000, 0.3994000000)
(8600.0000000000, 0.3993900000)
(8700.0000000000, 0.3991800000)
(8800.0000000000, 0.3989600000)
(8900.0000000000, 0.3988000000)
(9000.0000000000, 0.3995200000)
(9100.0000000000, 0.3984200000)
(9200.0000000000, 0.3977200000)
(9300.0000000000, 0.3990000000)
(9400.0000000000, 0.3991700000)
(9500.0000000000, 0.3987100000)
(9600.0000000000, 0.3983400000)
(9700.0000000000, 0.3994000000)
(9800.0000000000, 0.3989700000)
(9900.0000000000, 0.3984500000)
(10000.0000000000, 0.3976000000)
};
\addplot[line width=1, solid,color=cyan!80!black,] coordinates {
(100.0000000000, 0.1415600000)
(200.0000000000, 0.2138300000)
(300.0000000000, 0.3166200000)
(400.0000000000, 0.3698200000)
(500.0000000000, 0.3842300000)
(600.0000000000, 0.3881600000)
(700.0000000000, 0.3912100000)
(800.0000000000, 0.3943400000)
(900.0000000000, 0.3951800000)
(1000.0000000000, 0.3958800000)
(1100.0000000000, 0.3961100000)
(1200.0000000000, 0.3963200000)
(1300.0000000000, 0.3964600000)
(1400.0000000000, 0.3981500000)
(1500.0000000000, 0.3965800000)
(1600.0000000000, 0.3968000000)
(1700.0000000000, 0.3969600000)
(1800.0000000000, 0.3978700000)
(1900.0000000000, 0.3980800000)
(2000.0000000000, 0.3974300000)
(2100.0000000000, 0.3976800000)
(2200.0000000000, 0.3972500000)
(2300.0000000000, 0.3975800000)
(2400.0000000000, 0.3981300000)
(2500.0000000000, 0.3981100000)
(2600.0000000000, 0.3981200000)
(2700.0000000000, 0.3976300000)
(2800.0000000000, 0.3986800000)
(2900.0000000000, 0.3982600000)
(3000.0000000000, 0.3977600000)
(3100.0000000000, 0.3987800000)
(3200.0000000000, 0.3986600000)
(3300.0000000000, 0.3979700000)
(3400.0000000000, 0.3978300000)
(3500.0000000000, 0.3991200000)
(3600.0000000000, 0.3993600000)
(3700.0000000000, 0.3983600000)
(3800.0000000000, 0.3983900000)
(3900.0000000000, 0.3985400000)
(4000.0000000000, 0.3995500000)
(4100.0000000000, 0.3986300000)
(4200.0000000000, 0.3985800000)
(4300.0000000000, 0.3978800000)
(4400.0000000000, 0.3984600000)
(4500.0000000000, 0.3980700000)
(4600.0000000000, 0.3981300000)
(4700.0000000000, 0.3989200000)
(4800.0000000000, 0.3989000000)
(4900.0000000000, 0.3994800000)
(5000.0000000000, 0.3991000000)
(5100.0000000000, 0.3990600000)
(5200.0000000000, 0.3990900000)
(5300.0000000000, 0.3994900000)
(5400.0000000000, 0.3982600000)
(5500.0000000000, 0.3978200000)
(5600.0000000000, 0.3980700000)
(5700.0000000000, 0.3992400000)
(5800.0000000000, 0.3991700000)
(5900.0000000000, 0.3995200000)
(6000.0000000000, 0.3988200000)
(6100.0000000000, 0.3986200000)
(6200.0000000000, 0.3991800000)
(6300.0000000000, 0.3986200000)
(6400.0000000000, 0.3986700000)
(6500.0000000000, 0.3987500000)
(6600.0000000000, 0.3993400000)
(6700.0000000000, 0.3997300000)
(6800.0000000000, 0.3993700000)
(6900.0000000000, 0.3989500000)
(7000.0000000000, 0.3992200000)
(7100.0000000000, 0.3992000000)
(7200.0000000000, 0.3982800000)
(7300.0000000000, 0.3991100000)
(7400.0000000000, 0.3989300000)
(7500.0000000000, 0.3988700000)
(7600.0000000000, 0.3988000000)
(7700.0000000000, 0.3993800000)
(7800.0000000000, 0.3982900000)
(7900.0000000000, 0.3987800000)
(8000.0000000000, 0.3983800000)
(8100.0000000000, 0.3986400000)
(8200.0000000000, 0.3989800000)
(8300.0000000000, 0.3999500000)
(8400.0000000000, 0.3995300000)
(8500.0000000000, 0.3986000000)
(8600.0000000000, 0.3991100000)
(8700.0000000000, 0.3987500000)
(8800.0000000000, 0.3986100000)
(8900.0000000000, 0.3989500000)
(9000.0000000000, 0.3997900000)
(9100.0000000000, 0.3999900000)
(9200.0000000000, 0.3984900000)
(9300.0000000000, 0.3980500000)
(9400.0000000000, 0.3994000000)
(9500.0000000000, 0.3990100000)
(9600.0000000000, 0.3989500000)
(9700.0000000000, 0.3990900000)
(9800.0000000000, 0.3994700000)
(9900.0000000000, 0.3991000000)
(10000.0000000000, 0.3987000000)
};
\end{axis}\end{tikzpicture}

%% file: figures/RiverSwim_payoff.tex
\begin{tikzpicture}
\begin{axis}[
xlabel={$t$},
ylabel={$\E r_t$},
scaled x ticks = false,
x tick label style = {/pgf/number format/fixed},
scaled y ticks = false,
y tick label style = {/pgf/number format/fixed},
]
\addplot[line width=1, densely dotted,color=red,] coordinates {
(100.0000000000, 0.0004775650)
(200.0000000000, 0.0005000000)
(300.0000000000, 0.0005000000)
(400.0000000000, 0.0005000000)
(500.0000000000, 0.0005000000)
(600.0000000000, 0.0005000000)
(700.0000000000, 0.0005000000)
(800.0000000000, 0.0005000000)
(900.0000000000, 0.0005000000)
(1000.0000000000, 0.0005000000)
(1100.0000000000, 0.0005000000)
(1200.0000000000, 0.0005000000)
(1300.0000000000, 0.0005000000)
(1400.0000000000, 0.0005000000)
(1500.0000000000, 0.0005000000)
(1600.0000000000, 0.0005000000)
(1700.0000000000, 0.0005000000)
(1800.0000000000, 0.0005000000)
(1900.0000000000, 0.0005000000)
(2000.0000000000, 0.0005000000)
(2100.0000000000, 0.0005000000)
(2200.0000000000, 0.0005000000)
(2300.0000000000, 0.0005000000)
(2400.0000000000, 0.0005000000)
(2500.0000000000, 0.0005000000)
(2600.0000000000, 0.0005000000)
(2700.0000000000, 0.0005000000)
(2800.0000000000, 0.0005000000)
(2900.0000000000, 0.0005000000)
(3000.0000000000, 0.0005000000)
(3100.0000000000, 0.0005000000)
(3200.0000000000, 0.0005000000)
(3300.0000000000, 0.0005000000)
(3400.0000000000, 0.0005000000)
(3500.0000000000, 0.0005000000)
(3600.0000000000, 0.0005000000)
(3700.0000000000, 0.0005000000)
(3800.0000000000, 0.0005000000)
(3900.0000000000, 0.0005000000)
(4000.0000000000, 0.0005000000)
(4100.0000000000, 0.0005000000)
(4200.0000000000, 0.0005000000)
(4300.0000000000, 0.0005000000)
(4400.0000000000, 0.0005000000)
(4500.0000000000, 0.0005000000)
(4600.0000000000, 0.0005000000)
(4700.0000000000, 0.0005000000)
(4800.0000000000, 0.0005000000)
(4900.0000000000, 0.0005000000)
(5000.0000000000, 0.0005000000)
(5100.0000000000, 0.0005000000)
(5200.0000000000, 0.0005000000)
(5300.0000000000, 0.0005000000)
(5400.0000000000, 0.0005000000)
(5500.0000000000, 0.0005000000)
(5600.0000000000, 0.0005000000)
(5700.0000000000, 0.0005000000)
(5800.0000000000, 0.0005000000)
(5900.0000000000, 0.0005000000)
(6000.0000000000, 0.0005000000)
(6100.0000000000, 0.0005000000)
(6200.0000000000, 0.0005000000)
(6300.0000000000, 0.0005000000)
(6400.0000000000, 0.0005000000)
(6500.0000000000, 0.0005000000)
(6600.0000000000, 0.0005000000)
(6700.0000000000, 0.0005000000)
(6800.0000000000, 0.0005000000)
(6900.0000000000, 0.0005000000)
(7000.0000000000, 0.0005000000)
(7100.0000000000, 0.0005000000)
(7200.0000000000, 0.0005000000)
(7300.0000000000, 0.0005000000)
(7400.0000000000, 0.0005000000)
(7500.0000000000, 0.0005000000)
(7600.0000000000, 0.0005000000)
(7700.0000000000, 0.0005000000)
(7800.0000000000, 0.0005000000)
(7900.0000000000, 0.0005000000)
(8000.0000000000, 0.0005000000)
(8100.0000000000, 0.0005000000)
(8200.0000000000, 0.0005000000)
(8300.0000000000, 0.0005000000)
(8400.0000000000, 0.0005000000)
(8500.0000000000, 0.0005000000)
(8600.0000000000, 0.0005000000)
(8700.0000000000, 0.0005000000)
(8800.0000000000, 0.0005000000)
(8900.0000000000, 0.0005000000)
(9000.0000000000, 0.0005000000)
(9100.0000000000, 0.0005000000)
(9200.0000000000, 0.0005000000)
(9300.0000000000, 0.0005000000)
(9400.0000000000, 0.0005000000)
(9500.0000000000, 0.0005000000)
(9600.0000000000, 0.0005000000)
(9700.0000000000, 0.0005000000)
(9800.0000000000, 0.0005000000)
(9900.0000000000, 0.0005000000)
(10000.0000000000, 0.0005000000)
};
\addplot[line width=1, densely dashed,color=green!50!black,] coordinates {
(100.0000000000, 0.0001712600)
(200.0000000000, 0.0000810100)
(300.0000000000, 0.0002313850)
(400.0000000000, 0.0005369050)
(500.0000000000, 0.0024673400)
(600.0000000000, 0.0046944800)
(700.0000000000, 0.0076785100)
(800.0000000000, 0.0087050950)
(900.0000000000, 0.0082407600)
(1000.0000000000, 0.0086244050)
(1100.0000000000, 0.0089260550)
(1200.0000000000, 0.0101466700)
(1300.0000000000, 0.0113521000)
(1400.0000000000, 0.0115071300)
(1500.0000000000, 0.0108764550)
(1600.0000000000, 0.0097556700)
(1700.0000000000, 0.0102962800)
(1800.0000000000, 0.0108438050)
(1900.0000000000, 0.0119643800)
(2000.0000000000, 0.0120189200)
(2100.0000000000, 0.0129543050)
(2200.0000000000, 0.0147018200)
(2300.0000000000, 0.0153561300)
(2400.0000000000, 0.0151382450)
(2500.0000000000, 0.0151655050)
(2600.0000000000, 0.0163406750)
(2700.0000000000, 0.0156429600)
(2800.0000000000, 0.0167012300)
(2900.0000000000, 0.0152529850)
(3000.0000000000, 0.0186367350)
(3100.0000000000, 0.0174059950)
(3200.0000000000, 0.0180549800)
(3300.0000000000, 0.0168287200)
(3400.0000000000, 0.0216480200)
(3500.0000000000, 0.0194890800)
(3600.0000000000, 0.0255538300)
(3700.0000000000, 0.0223849250)
(3800.0000000000, 0.0317942850)
(3900.0000000000, 0.0329967250)
(4000.0000000000, 0.0362237150)
(4100.0000000000, 0.0393830700)
(4200.0000000000, 0.0344890850)
(4300.0000000000, 0.0414854600)
(4400.0000000000, 0.0392247950)
(4500.0000000000, 0.0391315300)
(4600.0000000000, 0.0377467850)
(4700.0000000000, 0.0375200200)
(4800.0000000000, 0.0394850000)
(4900.0000000000, 0.0355819200)
(5000.0000000000, 0.0355102800)
(5100.0000000000, 0.0391053150)
(5200.0000000000, 0.0346109100)
(5300.0000000000, 0.0358541100)
(5400.0000000000, 0.0378703000)
(5500.0000000000, 0.0354550950)
(5600.0000000000, 0.0428885000)
(5700.0000000000, 0.0376524100)
(5800.0000000000, 0.0387624950)
(5900.0000000000, 0.0433119550)
(6000.0000000000, 0.0406818600)
(6100.0000000000, 0.0401690200)
(6200.0000000000, 0.0428239850)
(6300.0000000000, 0.0412621650)
(6400.0000000000, 0.0423298800)
(6500.0000000000, 0.0460453100)
(6600.0000000000, 0.0424556050)
(6700.0000000000, 0.0424592250)
(6800.0000000000, 0.0447194050)
(6900.0000000000, 0.0420293600)
(7000.0000000000, 0.0406768800)
(7100.0000000000, 0.0451912150)
(7200.0000000000, 0.0438482750)
(7300.0000000000, 0.0415997850)
(7400.0000000000, 0.0472969900)
(7500.0000000000, 0.0465242150)
(7600.0000000000, 0.0418576350)
(7700.0000000000, 0.0464100750)
(7800.0000000000, 0.0449439250)
(7900.0000000000, 0.0436984150)
(8000.0000000000, 0.0450736500)
(8100.0000000000, 0.0472845700)
(8200.0000000000, 0.0453176950)
(8300.0000000000, 0.0434670150)
(8400.0000000000, 0.0496926900)
(8500.0000000000, 0.0459394800)
(8600.0000000000, 0.0444146450)
(8700.0000000000, 0.0481602100)
(8800.0000000000, 0.0485809400)
(8900.0000000000, 0.0471150350)
(9000.0000000000, 0.0478757400)
(9100.0000000000, 0.0513846700)
(9200.0000000000, 0.0480563700)
(9300.0000000000, 0.0453005100)
(9400.0000000000, 0.0475514500)
(9500.0000000000, 0.0510239700)
(9600.0000000000, 0.0472600500)
(9700.0000000000, 0.0464617350)
(9800.0000000000, 0.0497693400)
(9900.0000000000, 0.0498248550)
(10000.0000000000, 0.0490587550)
};
\addplot[line width=1.5, dotted,color=blue,] coordinates {
(100.0000000000, 0.0014987150)
(200.0000000000, 0.0036822550)
(300.0000000000, 0.0085647000)
(400.0000000000, 0.0168391750)
(500.0000000000, 0.0267164250)
(600.0000000000, 0.0372913450)
(700.0000000000, 0.0450790500)
(800.0000000000, 0.0505818500)
(900.0000000000, 0.0559374300)
(1000.0000000000, 0.0600650950)
(1100.0000000000, 0.0625043050)
(1200.0000000000, 0.0628840650)
(1300.0000000000, 0.0612327550)
(1400.0000000000, 0.0631531700)
(1500.0000000000, 0.0636119500)
(1600.0000000000, 0.0658510200)
(1700.0000000000, 0.0627527250)
(1800.0000000000, 0.0664822700)
(1900.0000000000, 0.0646417200)
(2000.0000000000, 0.0660323600)
(2100.0000000000, 0.0668628850)
(2200.0000000000, 0.0651330650)
(2300.0000000000, 0.0663611750)
(2400.0000000000, 0.0649613700)
(2500.0000000000, 0.0645609100)
(2600.0000000000, 0.0654615800)
(2700.0000000000, 0.0631431300)
(2800.0000000000, 0.0658912350)
(2900.0000000000, 0.0661007650)
(3000.0000000000, 0.0678215150)
(3100.0000000000, 0.0662312850)
(3200.0000000000, 0.0651908800)
(3300.0000000000, 0.0659503900)
(3400.0000000000, 0.0635212650)
(3500.0000000000, 0.0676008150)
(3600.0000000000, 0.0662511500)
(3700.0000000000, 0.0652926750)
(3800.0000000000, 0.0667817100)
(3900.0000000000, 0.0647320100)
(4000.0000000000, 0.0654809300)
(4100.0000000000, 0.0658907900)
(4200.0000000000, 0.0675611800)
(4300.0000000000, 0.0664007550)
(4400.0000000000, 0.0681506850)
(4500.0000000000, 0.0679503600)
(4600.0000000000, 0.0671306750)
(4700.0000000000, 0.0665006650)
(4800.0000000000, 0.0660705550)
(4900.0000000000, 0.0670704600)
(5000.0000000000, 0.0665611450)
(5100.0000000000, 0.0668011850)
(5200.0000000000, 0.0680715850)
(5300.0000000000, 0.0660112100)
(5400.0000000000, 0.0658115500)
(5500.0000000000, 0.0669819250)
(5600.0000000000, 0.0680412700)
(5700.0000000000, 0.0635410400)
(5800.0000000000, 0.0668615400)
(5900.0000000000, 0.0662412450)
(6000.0000000000, 0.0667815300)
(6100.0000000000, 0.0659012400)
(6200.0000000000, 0.0664614200)
(6300.0000000000, 0.0674221300)
(6400.0000000000, 0.0666014350)
(6500.0000000000, 0.0660116250)
(6600.0000000000, 0.0658512750)
(6700.0000000000, 0.0662708250)
(6800.0000000000, 0.0685412050)
(6900.0000000000, 0.0690609200)
(7000.0000000000, 0.0659815750)
(7100.0000000000, 0.0657007650)
(7200.0000000000, 0.0655908000)
(7300.0000000000, 0.0692713900)
(7400.0000000000, 0.0670813650)
(7500.0000000000, 0.0661007650)
(7600.0000000000, 0.0662504100)
(7700.0000000000, 0.0671815850)
(7800.0000000000, 0.0672513950)
(7900.0000000000, 0.0663106600)
(8000.0000000000, 0.0667001000)
(8100.0000000000, 0.0658302550)
(8200.0000000000, 0.0674001350)
(8300.0000000000, 0.0662602750)
(8400.0000000000, 0.0667110650)
(8500.0000000000, 0.0646117650)
(8600.0000000000, 0.0653316650)
(8700.0000000000, 0.0657908700)
(8800.0000000000, 0.0669109450)
(8900.0000000000, 0.0683717550)
(9000.0000000000, 0.0673506300)
(9100.0000000000, 0.0657404950)
(9200.0000000000, 0.0663508650)
(9300.0000000000, 0.0668405000)
(9400.0000000000, 0.0666808700)
(9500.0000000000, 0.0630215050)
(9600.0000000000, 0.0661011650)
(9700.0000000000, 0.0663106300)
(9800.0000000000, 0.0658309200)
(9900.0000000000, 0.0664807450)
(10000.0000000000, 0.0673708800)
};
\addplot[line width=1.5, dashed,color=magenta,] coordinates {
(100.0000000000, 0.0022355950)
(200.0000000000, 0.0071852350)
(300.0000000000, 0.0154205400)
(400.0000000000, 0.0239491350)
(500.0000000000, 0.0350889050)
(600.0000000000, 0.0444996950)
(700.0000000000, 0.0513464700)
(800.0000000000, 0.0544918950)
(900.0000000000, 0.0576086550)
(1000.0000000000, 0.0606565100)
(1100.0000000000, 0.0621466450)
(1200.0000000000, 0.0622050100)
(1300.0000000000, 0.0601856150)
(1400.0000000000, 0.0626268950)
(1500.0000000000, 0.0630447800)
(1600.0000000000, 0.0652549900)
(1700.0000000000, 0.0629645350)
(1800.0000000000, 0.0658533900)
(1900.0000000000, 0.0637040350)
(2000.0000000000, 0.0659635700)
(2100.0000000000, 0.0663929150)
(2200.0000000000, 0.0649534050)
(2300.0000000000, 0.0660333800)
(2400.0000000000, 0.0638838700)
(2500.0000000000, 0.0636536150)
(2600.0000000000, 0.0650347600)
(2700.0000000000, 0.0623643150)
(2800.0000000000, 0.0648432000)
(2900.0000000000, 0.0654336050)
(3000.0000000000, 0.0673021950)
(3100.0000000000, 0.0654620950)
(3200.0000000000, 0.0644833650)
(3300.0000000000, 0.0655619000)
(3400.0000000000, 0.0638523000)
(3500.0000000000, 0.0672123950)
(3600.0000000000, 0.0659730400)
(3700.0000000000, 0.0648321100)
(3800.0000000000, 0.0667317400)
(3900.0000000000, 0.0646337100)
(4000.0000000000, 0.0650737300)
(4100.0000000000, 0.0655527350)
(4200.0000000000, 0.0671016400)
(4300.0000000000, 0.0662424300)
(4400.0000000000, 0.0679525700)
(4500.0000000000, 0.0675926550)
(4600.0000000000, 0.0665319900)
(4700.0000000000, 0.0662620400)
(4800.0000000000, 0.0657724900)
(4900.0000000000, 0.0667120350)
(5000.0000000000, 0.0666825400)
(5100.0000000000, 0.0664327450)
(5200.0000000000, 0.0676729850)
(5300.0000000000, 0.0654027750)
(5400.0000000000, 0.0657929300)
(5500.0000000000, 0.0671915000)
(5600.0000000000, 0.0677913100)
(5700.0000000000, 0.0633128500)
(5800.0000000000, 0.0664316500)
(5900.0000000000, 0.0660721200)
(6000.0000000000, 0.0669715750)
(6100.0000000000, 0.0660311050)
(6200.0000000000, 0.0665106200)
(6300.0000000000, 0.0674214150)
(6400.0000000000, 0.0668116700)
(6500.0000000000, 0.0661110450)
(6600.0000000000, 0.0656913450)
(6700.0000000000, 0.0659315200)
(6800.0000000000, 0.0684414950)
(6900.0000000000, 0.0690315800)
(7000.0000000000, 0.0660114850)
(7100.0000000000, 0.0657313150)
(7200.0000000000, 0.0654213350)
(7300.0000000000, 0.0691517450)
(7400.0000000000, 0.0666229950)
(7500.0000000000, 0.0658407100)
(7600.0000000000, 0.0659905100)
(7700.0000000000, 0.0671622200)
(7800.0000000000, 0.0671319900)
(7900.0000000000, 0.0664208300)
(8000.0000000000, 0.0661227850)
(8100.0000000000, 0.0655032500)
(8200.0000000000, 0.0668126500)
(8300.0000000000, 0.0660224050)
(8400.0000000000, 0.0669122450)
(8500.0000000000, 0.0646305300)
(8600.0000000000, 0.0654710050)
(8700.0000000000, 0.0658310400)
(8800.0000000000, 0.0671011000)
(8900.0000000000, 0.0683827650)
(9000.0000000000, 0.0672005300)
(9100.0000000000, 0.0653311200)
(9200.0000000000, 0.0660417250)
(9300.0000000000, 0.0666825700)
(9400.0000000000, 0.0662809200)
(9500.0000000000, 0.0628513050)
(9600.0000000000, 0.0659916450)
(9700.0000000000, 0.0662620000)
(9800.0000000000, 0.0657110150)
(9900.0000000000, 0.0665307050)
(10000.0000000000, 0.0674313500)
};
\addplot[line width=1, solid,color=cyan!80!black,] coordinates {
(100.0000000000, 0.0014578300)
(200.0000000000, 0.0035354250)
(300.0000000000, 0.0072502300)
(400.0000000000, 0.0126574450)
(500.0000000000, 0.0184729100)
(600.0000000000, 0.0255039850)
(700.0000000000, 0.0335667550)
(800.0000000000, 0.0396770300)
(900.0000000000, 0.0446575250)
(1000.0000000000, 0.0497989900)
(1100.0000000000, 0.0506543000)
(1200.0000000000, 0.0546415700)
(1300.0000000000, 0.0541692500)
(1400.0000000000, 0.0561288250)
(1500.0000000000, 0.0602383150)
(1600.0000000000, 0.0634457300)
(1700.0000000000, 0.0598668200)
(1800.0000000000, 0.0607953300)
(1900.0000000000, 0.0607832950)
(2000.0000000000, 0.0613266000)
(2100.0000000000, 0.0614962150)
(2200.0000000000, 0.0631861400)
(2300.0000000000, 0.0634839250)
(2400.0000000000, 0.0624148700)
(2500.0000000000, 0.0616852400)
(2600.0000000000, 0.0624643650)
(2700.0000000000, 0.0634151100)
(2800.0000000000, 0.0616152350)
(2900.0000000000, 0.0620045450)
(3000.0000000000, 0.0640428000)
(3100.0000000000, 0.0626850600)
(3200.0000000000, 0.0646636900)
(3300.0000000000, 0.0661409700)
(3400.0000000000, 0.0643221400)
(3500.0000000000, 0.0643840450)
(3600.0000000000, 0.0630332450)
(3700.0000000000, 0.0635531150)
(3800.0000000000, 0.0661137400)
(3900.0000000000, 0.0640514250)
(4000.0000000000, 0.0660632000)
(4100.0000000000, 0.0664032850)
(4200.0000000000, 0.0641333850)
(4300.0000000000, 0.0649446100)
(4400.0000000000, 0.0652020750)
(4500.0000000000, 0.0644141000)
(4600.0000000000, 0.0643449950)
(4700.0000000000, 0.0648630100)
(4800.0000000000, 0.0663440350)
(4900.0000000000, 0.0636646700)
(5000.0000000000, 0.0659049750)
(5100.0000000000, 0.0642245550)
(5200.0000000000, 0.0641033400)
(5300.0000000000, 0.0667725400)
(5400.0000000000, 0.0648029600)
(5500.0000000000, 0.0638744050)
(5600.0000000000, 0.0669625650)
(5700.0000000000, 0.0648419650)
(5800.0000000000, 0.0661821450)
(5900.0000000000, 0.0654014950)
(6000.0000000000, 0.0654731600)
(6100.0000000000, 0.0636520000)
(6200.0000000000, 0.0662724050)
(6300.0000000000, 0.0639431650)
(6400.0000000000, 0.0650009500)
(6500.0000000000, 0.0645221450)
(6600.0000000000, 0.0658715400)
(6700.0000000000, 0.0647028200)
(6800.0000000000, 0.0661607850)
(6900.0000000000, 0.0648312950)
(7000.0000000000, 0.0672319600)
(7100.0000000000, 0.0644929400)
(7200.0000000000, 0.0653329550)
(7300.0000000000, 0.0668431500)
(7400.0000000000, 0.0650623500)
(7500.0000000000, 0.0652519600)
(7600.0000000000, 0.0654423500)
(7700.0000000000, 0.0650330250)
(7800.0000000000, 0.0667825650)
(7900.0000000000, 0.0652416250)
(8000.0000000000, 0.0661019300)
(8100.0000000000, 0.0654118300)
(8200.0000000000, 0.0658623250)
(8300.0000000000, 0.0657823100)
(8400.0000000000, 0.0654022550)
(8500.0000000000, 0.0664415250)
(8600.0000000000, 0.0648832550)
(8700.0000000000, 0.0660719200)
(8800.0000000000, 0.0660313600)
(8900.0000000000, 0.0643543150)
(9000.0000000000, 0.0653309700)
(9100.0000000000, 0.0669007700)
(9200.0000000000, 0.0666618300)
(9300.0000000000, 0.0648129050)
(9400.0000000000, 0.0661815900)
(9500.0000000000, 0.0663112900)
(9600.0000000000, 0.0653213400)
(9700.0000000000, 0.0655420450)
(9800.0000000000, 0.0671112650)
(9900.0000000000, 0.0656213100)
(10000.0000000000, 0.0670226350)
};
\end{axis}\end{tikzpicture}

%% file: figures/MountainCar_payoff.tex
\begin{tikzpicture}
\begin{axis}[
xlabel={$t$},
ylabel={$\E r_t$},
scaled x ticks = false,
x tick label style = {/pgf/number format/fixed},
scaled y ticks = false,
y tick label style = {/pgf/number format/fixed},
]
\addplot[line width=1, densely dotted,color=red,] coordinates {
(200.0000000000, -0.9989500000)
(300.0000000000, -0.9982000000)
(400.0000000000, -0.9985500000)
(500.0000000000, -0.9975800000)
(600.0000000000, -0.9982100000)
(700.0000000000, -0.9984800000)
(800.0000000000, -0.9978400000)
(900.0000000000, -0.9975200000)
(1000.0000000000, -0.9982200000)
(1100.0000000000, -0.9978100000)
(1200.0000000000, -0.9979100000)
(1300.0000000000, -0.9982500000)
(1400.0000000000, -0.9984300000)
(1500.0000000000, -0.9982400000)
(1600.0000000000, -0.9983800000)
(1700.0000000000, -0.9983100000)
(1800.0000000000, -0.9983400000)
(1900.0000000000, -0.9981300000)
(2000.0000000000, -0.9982700000)
(2100.0000000000, -0.9982200000)
(2200.0000000000, -0.9983300000)
(2300.0000000000, -0.9977500000)
(2400.0000000000, -0.9980800000)
(2500.0000000000, -0.9981000000)
(2600.0000000000, -0.9978100000)
(2700.0000000000, -0.9978500000)
(2800.0000000000, -0.9979300000)
(2900.0000000000, -0.9978500000)
(3000.0000000000, -0.9976700000)
(3100.0000000000, -0.9974000000)
(3200.0000000000, -0.9974700000)
(3300.0000000000, -0.9974300000)
(3400.0000000000, -0.9974200000)
(3500.0000000000, -0.9972000000)
(3600.0000000000, -0.9971400000)
(3700.0000000000, -0.9966800000)
(3800.0000000000, -0.9970900000)
(3900.0000000000, -0.9967000000)
(4000.0000000000, -0.9961300000)
(4100.0000000000, -0.9966400000)
(4200.0000000000, -0.9969700000)
(4300.0000000000, -0.9971800000)
(4400.0000000000, -0.9960800000)
(4500.0000000000, -0.9965300000)
(4600.0000000000, -0.9960300000)
(4700.0000000000, -0.9968300000)
(4800.0000000000, -0.9965800000)
(4900.0000000000, -0.9958100000)
(5000.0000000000, -0.9958400000)
(5100.0000000000, -0.9961700000)
(5200.0000000000, -0.9962300000)
(5300.0000000000, -0.9957500000)
(5400.0000000000, -0.9952600000)
(5500.0000000000, -0.9956200000)
(5600.0000000000, -0.9951900000)
(5700.0000000000, -0.9949500000)
(5800.0000000000, -0.9950900000)
(5900.0000000000, -0.9943100000)
(6000.0000000000, -0.9947500000)
(6100.0000000000, -0.9948300000)
(6200.0000000000, -0.9948700000)
(6300.0000000000, -0.9952300000)
(6400.0000000000, -0.9957700000)
(6500.0000000000, -0.9953500000)
(6600.0000000000, -0.9952200000)
(6700.0000000000, -0.9951600000)
(6800.0000000000, -0.9948800000)
(6900.0000000000, -0.9953000000)
(7000.0000000000, -0.9951400000)
(7100.0000000000, -0.9950600000)
(7200.0000000000, -0.9944700000)
(7300.0000000000, -0.9948500000)
(7400.0000000000, -0.9946800000)
(7500.0000000000, -0.9947400000)
(7600.0000000000, -0.9940000000)
(7700.0000000000, -0.9941500000)
(7800.0000000000, -0.9941200000)
(7900.0000000000, -0.9942000000)
(8000.0000000000, -0.9947200000)
(8100.0000000000, -0.9941700000)
(8200.0000000000, -0.9938300000)
(8300.0000000000, -0.9939000000)
(8400.0000000000, -0.9934400000)
(8500.0000000000, -0.9941200000)
(8600.0000000000, -0.9936900000)
(8700.0000000000, -0.9938700000)
(8800.0000000000, -0.9941400000)
(8900.0000000000, -0.9933000000)
(9000.0000000000, -0.9930400000)
(9100.0000000000, -0.9931500000)
(9200.0000000000, -0.9937400000)
(9300.0000000000, -0.9934700000)
(9400.0000000000, -0.9941900000)
(9500.0000000000, -0.9935300000)
(9600.0000000000, -0.9935300000)
(9700.0000000000, -0.9939200000)
(9800.0000000000, -0.9937700000)
(9900.0000000000, -0.9939700000)
(10000.0000000000, -0.9934600000)
};
\addplot[line width=1, densely dashed,color=green!50!black,] coordinates {
(200.0000000000, -0.9996700000)
(300.0000000000, -0.9973700000)
(400.0000000000, -0.9959900000)
(500.0000000000, -0.9982700000)
(600.0000000000, -0.9969700000)
(700.0000000000, -0.9967800000)
(800.0000000000, -0.9958700000)
(900.0000000000, -0.9957000000)
(1000.0000000000, -0.9947000000)
(1100.0000000000, -0.9949900000)
(1200.0000000000, -0.9944300000)
(1300.0000000000, -0.9944200000)
(1400.0000000000, -0.9948500000)
(1500.0000000000, -0.9935100000)
(1600.0000000000, -0.9937800000)
(1700.0000000000, -0.9942700000)
(1800.0000000000, -0.9944900000)
(1900.0000000000, -0.9941800000)
(2000.0000000000, -0.9950700000)
(2100.0000000000, -0.9944800000)
(2200.0000000000, -0.9952700000)
(2300.0000000000, -0.9955000000)
(2400.0000000000, -0.9944900000)
(2500.0000000000, -0.9945300000)
(2600.0000000000, -0.9952000000)
(2700.0000000000, -0.9945800000)
(2800.0000000000, -0.9939100000)
(2900.0000000000, -0.9939100000)
(3000.0000000000, -0.9937600000)
(3100.0000000000, -0.9937600000)
(3200.0000000000, -0.9924400000)
(3300.0000000000, -0.9934100000)
(3400.0000000000, -0.9935600000)
(3500.0000000000, -0.9918400000)
(3600.0000000000, -0.9928500000)
(3700.0000000000, -0.9932100000)
(3800.0000000000, -0.9933800000)
(3900.0000000000, -0.9930900000)
(4000.0000000000, -0.9921600000)
(4100.0000000000, -0.9925600000)
(4200.0000000000, -0.9922700000)
(4300.0000000000, -0.9930700000)
(4400.0000000000, -0.9920500000)
(4500.0000000000, -0.9929000000)
(4600.0000000000, -0.9928400000)
(4700.0000000000, -0.9926600000)
(4800.0000000000, -0.9937800000)
(4900.0000000000, -0.9921700000)
(5000.0000000000, -0.9926600000)
(5100.0000000000, -0.9936800000)
(5200.0000000000, -0.9934900000)
(5300.0000000000, -0.9941000000)
(5400.0000000000, -0.9937700000)
(5500.0000000000, -0.9937100000)
(5600.0000000000, -0.9944700000)
(5700.0000000000, -0.9948100000)
(5800.0000000000, -0.9949200000)
(5900.0000000000, -0.9942800000)
(6000.0000000000, -0.9940500000)
(6100.0000000000, -0.9948200000)
(6200.0000000000, -0.9952400000)
(6300.0000000000, -0.9957900000)
(6400.0000000000, -0.9953600000)
(6500.0000000000, -0.9952500000)
(6600.0000000000, -0.9953300000)
(6700.0000000000, -0.9951400000)
(6800.0000000000, -0.9955800000)
(6900.0000000000, -0.9955100000)
(7000.0000000000, -0.9960700000)
(7100.0000000000, -0.9962700000)
(7200.0000000000, -0.9970200000)
(7300.0000000000, -0.9966600000)
(7400.0000000000, -0.9971000000)
(7500.0000000000, -0.9976400000)
(7600.0000000000, -0.9971200000)
(7700.0000000000, -0.9964200000)
(7800.0000000000, -0.9967800000)
(7900.0000000000, -0.9968800000)
(8000.0000000000, -0.9970800000)
(8100.0000000000, -0.9971200000)
(8200.0000000000, -0.9976300000)
(8300.0000000000, -0.9971900000)
(8400.0000000000, -0.9977200000)
(8500.0000000000, -0.9976300000)
(8600.0000000000, -0.9976400000)
(8700.0000000000, -0.9976400000)
(8800.0000000000, -0.9977300000)
(8900.0000000000, -0.9976900000)
(9000.0000000000, -0.9978700000)
(9100.0000000000, -0.9977400000)
(9200.0000000000, -0.9977900000)
(9300.0000000000, -0.9975200000)
(9400.0000000000, -0.9979200000)
(9500.0000000000, -0.9979800000)
(9600.0000000000, -0.9982300000)
(9700.0000000000, -0.9980700000)
(9800.0000000000, -0.9982100000)
(9900.0000000000, -0.9986300000)
(10000.0000000000, -0.9985300000)
};
\addplot[line width=1.5, dotted,color=blue,] coordinates {
(200.0000000000, -0.9960500000)
(300.0000000000, -0.9962500000)
(400.0000000000, -0.9960700000)
(500.0000000000, -0.9970000000)
(600.0000000000, -0.9965200000)
(700.0000000000, -0.9964200000)
(800.0000000000, -0.9963900000)
(900.0000000000, -0.9968600000)
(1000.0000000000, -0.9964600000)
(1100.0000000000, -0.9961100000)
(1200.0000000000, -0.9961300000)
(1300.0000000000, -0.9957500000)
(1400.0000000000, -0.9957500000)
(1500.0000000000, -0.9945700000)
(1600.0000000000, -0.9949200000)
(1700.0000000000, -0.9948300000)
(1800.0000000000, -0.9952100000)
(1900.0000000000, -0.9946700000)
(2000.0000000000, -0.9938800000)
(2100.0000000000, -0.9945300000)
(2200.0000000000, -0.9942400000)
(2300.0000000000, -0.9934700000)
(2400.0000000000, -0.9935700000)
(2500.0000000000, -0.9924600000)
(2600.0000000000, -0.9934500000)
(2700.0000000000, -0.9928200000)
(2800.0000000000, -0.9934400000)
(2900.0000000000, -0.9918100000)
(3000.0000000000, -0.9919300000)
(3100.0000000000, -0.9923600000)
(3200.0000000000, -0.9923800000)
(3300.0000000000, -0.9921200000)
(3400.0000000000, -0.9902700000)
(3500.0000000000, -0.9900400000)
(3600.0000000000, -0.9897600000)
(3700.0000000000, -0.9897500000)
(3800.0000000000, -0.9894800000)
(3900.0000000000, -0.9900000000)
(4000.0000000000, -0.9895300000)
(4100.0000000000, -0.9874900000)
(4200.0000000000, -0.9871600000)
(4300.0000000000, -0.9879900000)
(4400.0000000000, -0.9865800000)
(4500.0000000000, -0.9856300000)
(4600.0000000000, -0.9853700000)
(4700.0000000000, -0.9836200000)
(4800.0000000000, -0.9848200000)
(4900.0000000000, -0.9833000000)
(5000.0000000000, -0.9831400000)
(5100.0000000000, -0.9810600000)
(5200.0000000000, -0.9803400000)
(5300.0000000000, -0.9817300000)
(5400.0000000000, -0.9799800000)
(5500.0000000000, -0.9799700000)
(5600.0000000000, -0.9805800000)
(5700.0000000000, -0.9791100000)
(5800.0000000000, -0.9783300000)
(5900.0000000000, -0.9768100000)
(6000.0000000000, -0.9784500000)
(6100.0000000000, -0.9780300000)
(6200.0000000000, -0.9766000000)
(6300.0000000000, -0.9777500000)
(6400.0000000000, -0.9758000000)
(6500.0000000000, -0.9760900000)
(6600.0000000000, -0.9764100000)
(6700.0000000000, -0.9758500000)
(6800.0000000000, -0.9750500000)
(6900.0000000000, -0.9752100000)
(7000.0000000000, -0.9749100000)
(7100.0000000000, -0.9733400000)
(7200.0000000000, -0.9737300000)
(7300.0000000000, -0.9730400000)
(7400.0000000000, -0.9746300000)
(7500.0000000000, -0.9732900000)
(7600.0000000000, -0.9735400000)
(7700.0000000000, -0.9717800000)
(7800.0000000000, -0.9729500000)
(7900.0000000000, -0.9723900000)
(8000.0000000000, -0.9712900000)
(8100.0000000000, -0.9707000000)
(8200.0000000000, -0.9706300000)
(8300.0000000000, -0.9717400000)
(8400.0000000000, -0.9716500000)
(8500.0000000000, -0.9719400000)
(8600.0000000000, -0.9706700000)
(8700.0000000000, -0.9702300000)
(8800.0000000000, -0.9721300000)
(8900.0000000000, -0.9700200000)
(9000.0000000000, -0.9716400000)
(9100.0000000000, -0.9696900000)
(9200.0000000000, -0.9703600000)
(9300.0000000000, -0.9693100000)
(9400.0000000000, -0.9706000000)
(9500.0000000000, -0.9687500000)
(9600.0000000000, -0.9684900000)
(9700.0000000000, -0.9678400000)
(9800.0000000000, -0.9676000000)
(9900.0000000000, -0.9696100000)
(10000.0000000000, -0.9685500000)
};
\addplot[line width=1.5, dashed,color=magenta,] coordinates {
(200.0000000000, -0.9962000000)
(300.0000000000, -0.9969000000)
(400.0000000000, -0.9967900000)
(500.0000000000, -0.9965400000)
(600.0000000000, -0.9963900000)
(700.0000000000, -0.9966400000)
(800.0000000000, -0.9963500000)
(900.0000000000, -0.9959900000)
(1000.0000000000, -0.9946100000)
(1100.0000000000, -0.9955100000)
(1200.0000000000, -0.9950300000)
(1300.0000000000, -0.9949400000)
(1400.0000000000, -0.9944000000)
(1500.0000000000, -0.9951300000)
(1600.0000000000, -0.9944400000)
(1700.0000000000, -0.9941500000)
(1800.0000000000, -0.9942400000)
(1900.0000000000, -0.9945100000)
(2000.0000000000, -0.9939700000)
(2100.0000000000, -0.9934000000)
(2200.0000000000, -0.9933400000)
(2300.0000000000, -0.9929600000)
(2400.0000000000, -0.9926000000)
(2500.0000000000, -0.9920400000)
(2600.0000000000, -0.9912900000)
(2700.0000000000, -0.9921300000)
(2800.0000000000, -0.9920200000)
(2900.0000000000, -0.9913100000)
(3000.0000000000, -0.9911500000)
(3100.0000000000, -0.9906000000)
(3200.0000000000, -0.9900500000)
(3300.0000000000, -0.9900200000)
(3400.0000000000, -0.9896200000)
(3500.0000000000, -0.9888500000)
(3600.0000000000, -0.9877600000)
(3700.0000000000, -0.9879800000)
(3800.0000000000, -0.9874100000)
(3900.0000000000, -0.9866600000)
(4000.0000000000, -0.9866100000)
(4100.0000000000, -0.9857700000)
(4200.0000000000, -0.9856500000)
(4300.0000000000, -0.9827900000)
(4400.0000000000, -0.9826400000)
(4500.0000000000, -0.9820000000)
(4600.0000000000, -0.9823600000)
(4700.0000000000, -0.9814500000)
(4800.0000000000, -0.9807000000)
(4900.0000000000, -0.9820900000)
(5000.0000000000, -0.9798700000)
(5100.0000000000, -0.9803200000)
(5200.0000000000, -0.9787200000)
(5300.0000000000, -0.9781200000)
(5400.0000000000, -0.9793100000)
(5500.0000000000, -0.9787300000)
(5600.0000000000, -0.9770600000)
(5700.0000000000, -0.9771600000)
(5800.0000000000, -0.9777600000)
(5900.0000000000, -0.9759200000)
(6000.0000000000, -0.9762200000)
(6100.0000000000, -0.9765000000)
(6200.0000000000, -0.9747900000)
(6300.0000000000, -0.9740000000)
(6400.0000000000, -0.9756300000)
(6500.0000000000, -0.9737400000)
(6600.0000000000, -0.9727400000)
(6700.0000000000, -0.9745500000)
(6800.0000000000, -0.9728000000)
(6900.0000000000, -0.9736300000)
(7000.0000000000, -0.9729500000)
(7100.0000000000, -0.9731200000)
(7200.0000000000, -0.9710200000)
(7300.0000000000, -0.9701600000)
(7400.0000000000, -0.9694300000)
(7500.0000000000, -0.9687800000)
(7600.0000000000, -0.9709900000)
(7700.0000000000, -0.9686100000)
(7800.0000000000, -0.9697700000)
(7900.0000000000, -0.9705900000)
(8000.0000000000, -0.9704000000)
(8100.0000000000, -0.9692700000)
(8200.0000000000, -0.9697100000)
(8300.0000000000, -0.9702800000)
(8400.0000000000, -0.9683000000)
(8500.0000000000, -0.9686900000)
(8600.0000000000, -0.9689000000)
(8700.0000000000, -0.9679600000)
(8800.0000000000, -0.9684300000)
(8900.0000000000, -0.9673200000)
(9000.0000000000, -0.9669800000)
(9100.0000000000, -0.9682700000)
(9200.0000000000, -0.9686000000)
(9300.0000000000, -0.9684800000)
(9400.0000000000, -0.9673700000)
(9500.0000000000, -0.9679900000)
(9600.0000000000, -0.9671700000)
(9700.0000000000, -0.9678800000)
(9800.0000000000, -0.9671100000)
(9900.0000000000, -0.9662600000)
(10000.0000000000, -0.9683700000)
};
\addplot[line width=1, solid,color=cyan!80!black,] coordinates {
(200.0000000000, -0.9970500000)
(300.0000000000, -0.9976300000)
(400.0000000000, -0.9973200000)
(500.0000000000, -0.9975900000)
(600.0000000000, -0.9979800000)
(700.0000000000, -0.9978700000)
(800.0000000000, -0.9977600000)
(900.0000000000, -0.9976100000)
(1000.0000000000, -0.9975100000)
(1100.0000000000, -0.9973300000)
(1200.0000000000, -0.9973400000)
(1300.0000000000, -0.9968400000)
(1400.0000000000, -0.9964100000)
(1500.0000000000, -0.9964100000)
(1600.0000000000, -0.9962500000)
(1700.0000000000, -0.9959300000)
(1800.0000000000, -0.9954400000)
(1900.0000000000, -0.9954100000)
(2000.0000000000, -0.9957300000)
(2100.0000000000, -0.9947900000)
(2200.0000000000, -0.9949200000)
(2300.0000000000, -0.9941200000)
(2400.0000000000, -0.9940300000)
(2500.0000000000, -0.9938300000)
(2600.0000000000, -0.9930600000)
(2700.0000000000, -0.9925100000)
(2800.0000000000, -0.9929700000)
(2900.0000000000, -0.9924300000)
(3000.0000000000, -0.9920800000)
(3100.0000000000, -0.9929200000)
(3200.0000000000, -0.9913800000)
(3300.0000000000, -0.9923500000)
(3400.0000000000, -0.9912800000)
(3500.0000000000, -0.9907600000)
(3600.0000000000, -0.9909700000)
(3700.0000000000, -0.9900800000)
(3800.0000000000, -0.9904800000)
(3900.0000000000, -0.9890200000)
(4000.0000000000, -0.9895700000)
(4100.0000000000, -0.9883300000)
(4200.0000000000, -0.9891100000)
(4300.0000000000, -0.9884100000)
(4400.0000000000, -0.9890100000)
(4500.0000000000, -0.9881500000)
(4600.0000000000, -0.9879000000)
(4700.0000000000, -0.9879200000)
(4800.0000000000, -0.9879000000)
(4900.0000000000, -0.9887400000)
(5000.0000000000, -0.9880300000)
(5100.0000000000, -0.9886200000)
(5200.0000000000, -0.9873100000)
(5300.0000000000, -0.9872000000)
(5400.0000000000, -0.9866900000)
(5500.0000000000, -0.9871100000)
(5600.0000000000, -0.9868400000)
(5700.0000000000, -0.9868300000)
(5800.0000000000, -0.9870000000)
(5900.0000000000, -0.9866300000)
(6000.0000000000, -0.9846000000)
(6100.0000000000, -0.9854900000)
(6200.0000000000, -0.9856800000)
(6300.0000000000, -0.9862900000)
(6400.0000000000, -0.9853300000)
(6500.0000000000, -0.9852800000)
(6600.0000000000, -0.9842600000)
(6700.0000000000, -0.9842900000)
(6800.0000000000, -0.9843400000)
(6900.0000000000, -0.9843600000)
(7000.0000000000, -0.9851900000)
(7100.0000000000, -0.9846000000)
(7200.0000000000, -0.9835600000)
(7300.0000000000, -0.9839900000)
(7400.0000000000, -0.9842700000)
(7500.0000000000, -0.9827300000)
(7600.0000000000, -0.9840900000)
(7700.0000000000, -0.9825600000)
(7800.0000000000, -0.9818800000)
(7900.0000000000, -0.9829600000)
(8000.0000000000, -0.9822600000)
(8100.0000000000, -0.9832500000)
(8200.0000000000, -0.9824300000)
(8300.0000000000, -0.9822100000)
(8400.0000000000, -0.9811500000)
(8500.0000000000, -0.9814800000)
(8600.0000000000, -0.9817100000)
(8700.0000000000, -0.9826900000)
(8800.0000000000, -0.9818800000)
(8900.0000000000, -0.9820200000)
(9000.0000000000, -0.9807500000)
(9100.0000000000, -0.9815100000)
(9200.0000000000, -0.9806500000)
(9300.0000000000, -0.9802100000)
(9400.0000000000, -0.9813900000)
(9500.0000000000, -0.9811300000)
(9600.0000000000, -0.9811400000)
(9700.0000000000, -0.9803100000)
(9800.0000000000, -0.9803900000)
(9900.0000000000, -0.9797500000)
(10000.0000000000, -0.9795500000)
};
\end{axis}\end{tikzpicture}